\documentclass{article}
\usepackage[utf8]{inputenc}
\usepackage{arxiv}

\usepackage{amsmath}
\usepackage{amssymb}
\usepackage{bm}
\usepackage{listings}
\newcommand{\norm}[1]{\left\lVert#1\right\rVert}
\DeclareMathOperator*{\argmin}{arg\,min}
\usepackage{mathtools}
\usepackage{xcolor}

\usepackage{ntheorem}
\theoremstyle{break}
\newtheorem{lemma}{Lemma}

\usepackage{changepage}   % for the adjustwidth environment
\newenvironment{proof}{\begin{adjustwidth}{1cm}{1cm}}{\end{adjustwidth}}

\usepackage[toc,page]{appendix}
\usepackage{tabularx,longtable,multirow,subfigure,caption}
\usepackage{hyperref}
\usepackage{booktabs}
\usepackage{lineno}

\newcommand{\fref}[1]{Figure~\ref{#1}}

\date{}
\begin{document}

\title{Attention-based Convolutional Autoencoders for 3D-Variational Data Assimilation}

\author{ \href{https://orcid.org/0000-0002-7213-5508}{\hspace{1mm}Julian Mack} \\
	Data Science Institute\thanks{Data Science Institute, Imperial College London, UK} \\
	\texttt{julianfmack@gmail.com} \\
	\And
	{\hspace{1mm}Rossella Arcucci}\thanks{Corresponding author} \\
	Data Science Institute\\
	\texttt{r.arcucci@imperial.ac.uk} \\
	\And
	{\hspace{1mm}Miguel Molina-Solana}\\
	Data Science Institute\\
	Universidad de Granada\thanks{Dept. Computer Science and AI, Universidad de Granada, Spain}\\
	\texttt{mmolinas@ic.ac.uk} \\
	\And
	{\hspace{1mm}Yi-Ke Guo}\\
	Data Science Institute\\
	\texttt{y.guo@imperial.ac.uk} \\
	}
% \author[DSI]{Julian Mack}\ead{julianfmack@gmail.com}
% \author[DSI]{Rossella Arcucci\corref{mycorrespondingauthor}}\ead{r.arcucci@imperial.ac.uk}
% \author[DSI,UGR]{Miguel Molina-Solana}\ead{mmolinas@ic.ac.uk}
% \author[DSI]{Yi-Ke Guo}\ead{y.guo@imperial.ac.uk}

% \cortext[mycorrespondingauthor]{Corresponding author}
% \address[DSI]{Data Science Institute, Imperial College London, UK}
% \address[UGR]{Dept. Computer Science and AI, Universidad de Granada, Spain}

\maketitle

%%%%%%%%%%%%%%%%%%%%%%%%%%%%%%%%%%%%
\begin{abstract}
We propose a new `Bi-Reduced Space' approach to solving 3D Variational Data Assimilation using Convolutional Autoencoders. We prove that our approach has the same solution as previous methods but has significantly lower computational complexity; in other words, we reduce the computational cost without affecting the data assimilation accuracy. We tested the new method with data from a real-world application: a pollution model of a site in Elephant and Castle, London and found that we could reduce the size of the background covariance matrix representation by $\mathcal{O}(10^3)$ and, at the same time, increase our data assimilation accuracy with respect to existing reduced space methods. 
\end{abstract}

\keywords{Variational Data Assimilation \and Attention Networks \and Convolutional Autoencoders}

\linenumbers
\nolinenumbers
%%%%%%%%%%%%%%%%%%%%%%%%%%%%%%%%%%%%
\section{Introduction}

Data Assimilation (DA) is an uncertainty quantification technique in which observation data and a forecasting model are used in tandem to generate predictions that are more accurate than those that would be produced using either component independently. DA is computationally costly for large systems~\cite{MetOffice2019a} and under operational constraints, it is often necessary to solve the problem in a reduced space in order to achieve real-time assimilation. 
In most relevant DA operational software, a variable transformation is performed on the variational functional to reduce the computational cost needed for computing the covariance matrix explicitly; to reduce the space, only Empirical Orthogonal Functions (EOFs) of the first largest eigenvalues of the error covariance matrix are considered. Since its introduction to meteorology by Edward Lorenz, EOFs analysis, which is essentially based on a TSVD, has become a fundamental tool in computational fluid dynamic modelling for data diagnostics and dynamical model reduction. Real world applications of TSVD (EOFs) basically exploit the fact that these methods allow a decomposition of a data function into a set of orthogonal functions, which are designed so that only a few of these functions are needed in lower-dimensional approximations. Nevertheless, the accuracy of the solution obtained by truncating, exhibits a high sensitivity to the variation of the value of the truncation parameter \cite{arcucci2017variational,arcucci2018toward}, so that a suitable truncation parameter is needed. This is a severe drawback of truncation-based methods and limits the utility of operational software based on these methods.
Convolutional Autoencoders (CAEs) have had huge successes in computer vision~\cite{Vincent2008,Lore2017} and particularly in image compression~\cite{Theis2017,Balle2018}. In this work, we use CAEs to produce a reduced space in which DA can be performed efficiently.

The structure of this paper is as follows: in Section~\ref{sec:background} we cover related work and we present the contribution of the present work. Section~\ref{sec:def} provides preliminary concepts and definitions, and Section~\ref{sec:contrib_theory} introduces our theoretical contribution. As the success of our approach is heavily conditioned on the choice of a CAE architecture, in Section~\ref{sec:architecture_search} we summarise the results of our extensive architecture search before evaluating our approach against existing VarDA methods in Section~\ref{sec:expt_TSVDvsAE}. In Section~\ref{sec:discussion} we discuss these results before concluding in Section~\ref{sec:conc_fw}.

%%%%%%%%%%%%%%%%%%%%%%%%%%%%%%%%%%%%
\section{Related work and contribution of the present work} \label{sec:background}
Forecasting models introduce uncertainty from numerous sources. These include, but are not limited to, uncertainty in initial conditions, imperfect representations of the underlying physical processes and numerical errors. As a result, a model without access to real-time data will accumulate errors until its predictions no longer correspond to reality~\cite{Tribbia2004}. Similarly, all observations will have an irreducible uncertainty as a result of imperfect measuring devices. The key idea in DA is that the overall uncertainty in a forecast can be reduced by producing a weighed average of model forecasts and observations. The canonical application of Data Assimilation (DA)~\cite{Lorenc1986,Lorenc1988} is Numerical Weather Prediction (NWP)~\cite{Courtier1994,Courtier1998,Huang2004a} but the technique has been utilised in contexts as diverse as oceanic modelling \cite{Evensen2003,Dobricic2008}, solar wind prediction \cite{Lang2019} and inner city pollution modelling \cite{Arcucci2019a,arcucci2018effective}. Our proposed approach is agnostic to the details of the forecasting model (i.e.\ it is non-intrusive) and is therefore applicable to any DA problem in which a reduced order system is used. 

Our proposed formulation of Variational DA \cite{Courtier1998,Huang2004a} extends the incremental formulation \cite{Courtier1994}. In 1992, Parrish et al. \cite{Parrish1992} proposed using a Control Variable Transform (CVT) to reduce the space of the background error covariance matrix $\bm{B}$ by performing Cholesky factorisation as $\bm{B}=\bm{V}\bm{V}^T$. Since then, many authors have used eigenanalysis techniques such as PCA or TSVD to reduce the rank of $\bm{V}$ \cite{Arcucci2019a}. In this work, we propose replacing these eigenanalysis approaches with a CAE that learns to compress $\bm{V}$ more efficiently, and with less information-loss than the removal of eigen-modes. 

This work builds on a previous publication \cite{Arcucci2019a} in which TSVD was used to precondition $\bm{V}$. The original authors used a test-site location in South London and synthetic data generated by Fluidity, an open-source finite-element fluid dynamic software (\url{http://fluidityproject.github.io/}). We test the proposed approach on the same domain and data to enable a clear comparison between the approaches. We find that our method gives considerably more accurate predictions and, in most cases, provides them sooner than the previous approach. In fact, our method is also more accurate (and much faster) than the CVT formulation of Parrish et al. \cite{Parrish1992}.

In this paper we make the following contributions:
\begin{enumerate}
    \item We propose a new `Bi-reduced space' 3D Variational Data Assimilation (3D-VarDA) formulation that has an online complexity that is independent of the number of assimilated observations (i.e.\ it can be used with arbitrarily dense sensor networks). We show that our approach has lower online complexity than \cite{Arcucci2019a} while also giving equivalent forecasts.
    \item We create and evaluate 3D extensions of a range of state-of-the-art CAEs for 2D image compression. To our knowledge, we are the first to extend the image compression network of \cite{Zhou2019} and image restoration GRDN of \cite{Kim2019} to three-dimensions. We find that Zhou et al.'s attention-based model \cite{Zhou2019} performs best, and make some small improvements to this system including the replacement of vanilla residual blocks \cite{He2016} with `NeXt' residual blocks \cite{Xie2017a} in order to reduce decoder inference time.
    \item This adapted CAE, in combination with our proposed DA formulation, achieves a substantial relative reduction in DA error of 37\% compared with the Arcucci et al. TSVD approach \cite{Arcucci2019a}. Depending on the number of assimilated observations, the proposed method is up x30 faster. We discuss the speed-accuracy tradeoff at length in Section~\ref{sec:expt_TSVDvsAE}. 
    \item We release a well tested open-source Python module \texttt{VarDACAE} that enables users to easily replicate our experiments, use our model implementations, and train CAEs for any Variational data assimilation problem. The repository can be found at \url{https://github.com/julianmack/Data_Assimilation}.
\end{enumerate}

\section{Preliminary Definitions} \label{sec:def}
In this section we define the key quantities for DA. In most cases, we follow the notation in Banister's review paper \cite{Bannister2017a}. 
\begin{itemize}
    \item Let $\bm{x}_t$ represent the state of the model at time $t$ such that: 
    \begin{equation}
        \bm{x}_t \in \mathbb{R}^{n}
    \end{equation}
    where $n$ is the number of elements in the model state vector. The state for $T$ time-steps can be given in a single matrix: \begin{equation}
        \bm{X} = [\bm{x}_0, \bm{x}_1 , ..., \bm{x}_T] \ \ \in \ \ \mathbb{R}^{n \times T}
    \end{equation}
    In most practical problems, $n$ is large and of order $ \geq \mathcal{O}(10^6)$. 
    \item Let $\bm{y}_t$ represent the observation space of the system where: \begin{equation}
        \bm{y}_t \in \mathbb{R}^{M}
    \end{equation}
    where typically $M << n$. The Met Office uses $M =0.01n$ \cite{MetOffice2019a}. 
    \item Let $\bm{\mathcal{H}}_t$ be an observation operator such that: \begin{equation}\bm{\mathcal{H}}_t[\bm{x}_t] = \bm{y}_t + \bm{\epsilon}_t\end{equation} where $\bm{\epsilon}_t \sim \mathcal{N}(\bm{0}, \bm{R}_t)$ is the observation error. Often, observations are assumed to be uncorrelated meaning that $\bm{R}_t$ is diagonal. When all observations are of the same type and made with the same device we have: 
    \begin{equation}
        \label{eq:R_def}
        \bm{R}_t = \sigma_0^2\bm{I}
    \end{equation}
    \item Let $\bm{\mathcal{M}}_{t-1,t}$ be the forecast model that propagates the system forward from time-step $t-1$ to $t$ such that: \begin{equation}\bm{x}_t = \bm{\mathcal{M}}_{t-1,t}[\bm{x}_{t-1}] + \bm{\eta}_t\end{equation} 
    where $\bm{\eta}_t \sim \mathcal{N}(\bm{0}, \bm{Q}_t)$ is the model error introduced over this interval. 
    \item Let $\bm{x}^b_t \in \mathbb{R}^n$ be the background state at time-step $t$. All \textit{a priori} information about the system is introduced through the first background state $\bm{x}^b_0$ and the model $\bm{\mathcal{M}}_{t-1,t}$. Future background estimates are then defined according to the free-running model in which $\bm{\eta}_t$ is assumed to be zero and therefore:
    \begin{equation}
        \bm{x}_t^b = \bm{\mathcal{M}}_{t-1,t}[\bm{x}^b_{t-1}]
    \end{equation}
    \item Let $\bm{B}_t$ represent the background state $\bm{x}^b_t$ covariance matrix. In theory, it is found by evaluating:
    \begin{equation} \label{eq:B_theory}
        \bm{B}_t = (\bm{x}^b_t - \bm{x^*}_t)(\bm{x}^b_t - \bm{x^*}_t)^T
    \end{equation} \\
    where $\bm{x^*}_t$ is the true state of the atmosphere. In practice, even if $\bm{x^*}_t$ were known, this matrix is too large to fit in memory as it has $\mathcal{O}(n^2)$ parameters which is $\geq \mathcal{O}(10^{12})$ for most practical problems.
\end{itemize}

\subsection{Variational DA, VarDA} \label{sec:VarDA}
VarDA involves minimising a cost function in order to find the most likely state values $\bm{X}^{DA}$  given the observations $\bm{y}_t$, the model predictions and their uncertainties. The problem is to find the initial state $\bm{x}_0^{DA}$ that satisfies:
\begin{equation}
    \bm{x}_0^{DA} = \argmin_{\bm{x}_0} J(\bm{x}_0)
\end{equation}
\begin{equation}\label{eq:cost_4d}
\begin{split}
    J(\bm{x}_0) = \frac{1}{2}\norm{ \bm{x}_0 - \bm{x}^b_0 }^2_{\bm{B}^{-1}_0} + \frac{1}{2}\sum^T_{t=0}\norm{\bm{y}_t - \bm{\mathcal{H}}_t[\bm{x}_t]}^2_{\bm{R}_t^{-1}}&  
\end{split}
\end{equation}

The first term in cost function $J(\bm{x}_0)$ measures the difference between the initial model state and our \textit{a priori} expectation of this state. The second term encodes the difference between the observations and the model forecasts. The cost function is explicitly differentiated and then approximately solved by first-order minimisation routines. Note that we are assuming all errors are Gaussian by using this least-squares formulation\footnote{This formulation has an equivalent solution to the Kalman Filter approach.}. This formulation is known as the `strong constraint 4D-Var' where:
\begin{itemize}
    \item `4D' refers to the fact that we are considering three spatial dimensions as well as one temporal dimension. It is contrasted with 3D-Var in which a single time-step is assimilated.
    \item `Strong' refers to the fact that model errors are assumed to be zero \cite{Zupanski1997}\footnote{The more general 4D-Var weak constraint incremental formulation is given in \cite{Courtier1994} and reviewed in a modern context in \cite{Bannister2017a}.}.
\end{itemize}
In the current work we are considering the 3D case where the cost function is:
\begin{equation}\label{eq:cost_3d}
    J(\bm{x}) = \frac{1}{2}\norm{ \bm{x}- \bm{x}^b }^2_{\bm{B}^{-1}} + \frac{1}{2}\norm{\bm{y} - \bm{\mathcal{H}}[\bm{x}]}^2_{\bm{R}^{-1}}
\end{equation}
Note that we have dropped the $t$ subscripts as we are only assimilating a single time-step. 

\subsubsection{Incremental VarDA}
If $\bm{\mathcal{H}}_{t}$ and $\bm{\mathcal{M}}_{t-1,t}$ are linear (which they are not in general), the cost functions (\ref{eq:cost_4d}) and (\ref{eq:cost_3d}) are convex. We can approximately linearize these operators about the background state $\bm{x}^b$ by formulating the problem in terms of perturbations to this state in a method known as the incremental formulation \cite{Courtier1994}:
\begin{equation}
\delta\bm{x} \coloneqq  \bm{x} - \bm{x}^b 
\end{equation} 
The problem statement then becomes:
\begin{align}
\begin{split}
    \delta&\bm{x}^{DA} = \argmin_{ \delta\bm{x}} J(\delta\bm{x}) \\ \label{eq:3d_incr}
     J(\delta\bm{x})  = 
     \frac{1}{2}&  \delta\bm{x}^T\bm{B}^{-1}\delta\bm{x} + \frac{1}{2}(\bm{d} - \bm{H}\delta\bm{x})^T \bm{R}^{-1}(\bm{d} - \bm{H}\delta\bm{x})
\end{split}
\end{align}

where $\bm{H}$ is the observation operator linearized about the background state and the `misfit' between observation and expected observation is:
\begin{equation}\label{eq:misfitH}
\bm{d} = \bm{y} - \bm{H}\bm{x}^b
\end{equation}

\subsubsection{Control Variable Transform}
As $\bm{B}$ is in $\mathcal{O}(n^2)$ it must be represented implicitly. A common way of doing this is by using the formulation proposed in \cite{Parrish1992} which states that:
\begin{align}
\delta\bm{x}  &= \bm{V} \bf{w}\label{eq:V_decode} \\
\bm{B} &= \bm{V V}^T \label{eq:b_sqrt}
\end{align}
where $ \bm{V}$ is the Cholesky factorisation of  $\bm{B}$. In this case the problem can be written as:
\begin{align} \label{eq:3d_preconditioned}
    \begin{split}
        & \mathbf{w}^{DA} = \argmin_{\mathbf{w}} J(\bf{w}) \\
    J(\mathbf{w})  &= 
     \frac{1}{2} \mathbf{w}^T\mathbf{w} + \frac{1}{2}(\bm{d} - \bm{H}\bm{V}\mathbf{w} )^T \bm{R}^{-1} (\bm{d} - \bm{H}\bm{V}\mathbf{w} )
    \end{split}
\end{align}
Following (\ref{eq:B_theory}), we can see that $\bm{V}$ is theoretically found by stacking a series of background states $\bm{X}^b$ and subtracting the true state, $\bm{x^*}$:
\begin{equation}
    \bm{V} = (\bm{X}^b - \bm{x^*})
\end{equation}
In reality, we do not know $\bm{x^*}$ but we can estimate $\bm{V}$ with a sample of $S$ model state forecasts $\bm{X}^b$ that we set aside as `background' such that:
\begin{equation*}
    \bm{X}^b = [\bm{x}_0^b, \bm{x}_1^b, ..., \bm{x}_S^b]  \in \mathbb{R}^{n \times S}
\end{equation*}
\begin{equation}
    \bm{V} = (\bm{X}^b - \bm{ x}^b) \ \in \mathbb{R}^{n \times S}
\end{equation}
where $\bm{ x}^b \coloneqq  \bar{\bm{ x}}^b $ is the mean of the sample of background states. In the incremental 3D-VarDA formulation this gives:
\begin{equation} \label{eq:V_full}
    \bm{V}= [\delta \bm{x}_0^b,  \ \delta \bm{x}_1^b,  \ ...,  \ \delta \bm{x}_S^b]  \in \mathbb{R}^{n \times S}
\end{equation}
where $\delta \bm{x}_i^b = \bm{x}_i -  \bm{ x}^b$. Note that $\textbf{w} \in \mathbb{R}^S$ and since $S << n$ in all practical cases, $\bf{w}$ is referred to as the `reduced space'. $\bm{V}$ is an affine transform from the reduced space to the full space and is used to obtain the assimilated state after minimising $J(\mathbf{w})$:
\begin{equation}
    \delta\bm{x} = \bm{V}\bf{w}
\end{equation} 

The problem in (\ref{eq:3d_preconditioned}) is poorly conditioned because, in most practical contexts, the matrix $\bm{V}$ is poorly conditioned\footnote{In other words, the ratio of $\bm{V}$'s largest to smallest eigenvalue is large.}.

\subsection{Truncated SVD}\label{sec:svd_background}
One way to precondition $\bm V$ and increase the speed of convergence is to use an eigenanalysis technique such as SVD \cite{Chai2007,Cheng2010} or PCA to generate Empirical Orthogonal Functions (EOFs) \cite{Lorenz1956} and remove low-variance modes from $\bm{V}$. We have implemented 3D-VarDA with TSVD as described in \cite{Arcucci2019a} and we evaluate the success of our proposed approach against this scheme in Section~\ref{sec:expt_TSVDvsAE}. In order to draw out the theoretical differences between the methods we briefly summarise TSVD here. \\
The singular value decomposition of matrix $\bm V$ is as follows:
\begin{equation}
    \bm V = \bm U \bm \Sigma \bm W^T
\end{equation}
where $\bm U \in \mathbb{R}^{n \times S}$, $\bm \Sigma \in \mathbb{R}^{S \times S}$ are both orthogonal and $\bm{W} \in \mathbb{R}^{n \times S} $ is diagonal and contains $\bm V$'s eigenvalues $\sigma_i$ such that:
\begin{align}
  \Sigma = \text{diag} [\sigma_1, \sigma_2, ... , \sigma_S]
\end{align}
where the eigenvalues appear in decreasing order:
\begin{equation}
\sigma_1 > \sigma_2 > ... > \sigma_S > 0
\end{equation}
To perform TSVD, $\tau$ of the modes are retained where $0 < \tau < S$ and the matrix $\bm V_{\tau}$ is reconstructed such that:
\begin{align}
\Sigma_{\tau} &= \text{diag} [\sigma_1, ... , \sigma_{\tau}, 0, ..., 0] \\
\bm {V}_{\tau} &= \bm U \bm \Sigma_{\tau} \bm W^T
\end{align}

This has generalised inverse:
\begin{equation}
    \bm {V}_{\tau}^+ = \bm W \bm \Sigma^+_{\tau} \bm U^T
\end{equation}
where $\bm \Sigma_{\tau}^+$ is the generalised inverse of $\bm \Sigma_{\tau}$ such that:
\begin{equation}
    \bm \Sigma_{\tau} = \text{diag} \bigg[\frac{1}{\sigma_1}, ... , \frac{1}{\sigma_{\tau}}, 0, ..., 0\bigg]
\end{equation}

\subsection{Autoencoders} \label{sec:AE_defs} \label{sec:background_AEs}
AEs are a self-supervised machine learning method first proposed in \cite{Rumelhart1986}. They consist of two components: an encoder $f(\bm{x})$ which compresses the input $\bm{x}$ to a small representation $\mathbf{z}$ (the AE's reduced space is referred to as the `latent space' in the machine learning literature), and a decoder $g(\mathbf{z})$ which reconstructs the input:
\begin{align} \label{eq:AE_def}
    \begin{split}
        f(\bm{x}) &= \mathbf{z} \\
    g(\mathbf{z}) &= \hat{\bm{x}}
    \end{split}
\end{align}
where the input data $\bm{x}$ and reconstruction $\hat{\bm{x}}$ are both $\in \mathbb{R}^n $ and the latent representation $\mathbf{z} \in \mathbb{R}^m$. As $m < n$, the `information bottleneck' forces the AE to find and exploit redundancies in the training data in order to implicitly model the data distribution\footnote{Although they are rarely thought of in this way, AEs are theoretically equivalent to a clustering algorithm \cite{Baldi2012} in the sense that the encoder learns to map commonly co-occurring inputs to a single internal representation.}. PCA is a special case of an AE with linear activations \cite{Baldi1989}.

Typically AEs are trained with the L1 or L2 reconstruction error. In this work we found that networks trained with the L2 loss $I_2(\bm{x}, \hat{\bm{x}})$ performed consistently better for DA than those trained with $I_1(\bm{x}, \hat{\bm{x}})$. We experimented with fine-tuning our models with the L1 loss as suggested in \cite{Lu2019} but this did not provide a consistent improvement.

AEs are useful in any problems in which a latent data representation is required for a downstream task but they have had success as a standalone solution in anomaly detection \cite{Sakurada2014,Baur2019} and image denoising \cite{Vincent2008,Lore2017}. There are many variants of Autoencoder but we have focused on the image compression CAE literature to guide our AE design. Variational AutoEncoders \cite{Kingma2013} have achieved recent success in a range of generative modelling tasks, particularly NLP \cite{Kusner2017,Pu2016}, and enforce the latent orthogonality condition that we will discuss in Section~\ref{sec:contrib_theory} but they typically produce poorer reconstructions than vanilla AEs (see chapter 20 of \cite{deeplearningbook2017}) so we do not use them in this work. Similarly, Huang et al. proposed a GAN-based AEs to generate visually plausible reconstructions \cite{Huang2019} but there is a lot of work to be done in proving the `correctness' of GAN-generated samples and we were not confident that `visually plausible' in the image domain would translate to scientifically correct in the 3D field domain. Nevertheless, this is certainly another route for future work.

The majority of CAE architectures in the DA literature are in the field of Reduced Order Modelling (ROM) \cite{Wang2017,Merwe2007,Wang2016,Loh2018}. ROMs are only applicable in 4D-VarDA when an online method of generating model forecasts is required. Their architecture reflects this: as this is a sequence-to-sequence problem \cite{Sutskever2014} most modern ROMs use LSTMs or other RNN variants to make their predictions. The result is that these architectures are not applicable in this work. Instead, we have reviewed the image compression literature and implemented a range of state of the art networks from this field. In the remainder of this section we describe the image compression problem and the SOTA networks for this task. The details and results of our architecture search are deferred to Section~\ref{sec:architecture_search}.

\subsubsection{AEs for Image Compression} \label{sec:ImageCompression}

\begin{table}[!htb]
    \let\center\empty
    \let\endcenter\relax
    \centering
    \resizebox{\textwidth}{!}{\begin{tabular}{cccccccccccc}
\toprule
\multicolumn{4}{c}{ \large{\textbf{System details}} }  & \multicolumn{7}{c}{ \large{\textbf{Building Blocks}} } & \large{\textbf{Equivalent model in} }\\
\textbf{Year} & \textbf{CLIC pos.} & \textbf{Authors/Team} & \textbf{Paper }& \textbf{GDNs}\cite{Balle2015}& \textbf{Parallel Filters} \cite{Szegedy2015}& \textbf{Multi-scale} \cite{Ronneberger2015}& \textbf{Attention} \cite{Bahdanau2014}  & \textbf{RDB} \cite{Zhang2020} & \textbf{CBAM} \cite{Woo2018}   & \textbf{RAB} \cite{Zhang2019}& \large{\textbf{section \ref{sec:expt_architecture}}}\\\midrule
2017 & -& Theis et al.  & \cite{Theis2017}  & $\times$& $\times$& $\times$& $\times$& $\times$& $\times$& $\times$& $\sim$  Backbone \\
2018 & -& Mentzer et al.& \cite{Mentzer2018} & $\times$& $\times$& $\times$& $\times$& $\times$& $\times$& $\times$& ResNeXt3-27-1-vanilla\\
2018 & -& Cheng et al.  & \cite{Cheng2018}  & $\times$& \checkmark& $\times$& $\times$& $\times$& $\times$& $\times$& ResNeXt3-3-N-vanilla$^{\mathsection}$ \\
2018 & -& Balle et al.& \cite{Balle2018}  & \checkmark & $\times$& $\times$& $\times$& $\times$& $\times$& $\times$& $\sim$  Tucodec (no RAB)  \\\midrule
2018 & 1st  & Tucodec& \cite{Zhou2018}& \checkmark & \checkmark& \checkmark  & $\times$& $\times$& $\times$& $\times$& Tucodec (no RAB)\\
2018 & 2nd  & iipTiramisu& \cite{Chen2018}& $\times$& $\times$& $\times$& $\times$& \checkmark & $\times$& $\times$& RBD3NeXt-8-1-vanilla \\
2018 & 2nd  & AmnesiackLite\textsuperscript{*} & -& - & -& -  & - & - & - & - & - \\
2018 & 3rd  & ZTESmartVideo\textsuperscript{*} & -& - & -& -  & - & - & - & - & - \\
2018 & 4th  & yfan & \cite{Fan2018}& $\times$& \checkmark& \checkmark  & $\times$& $\times$& $\times$& $\times$& ResNeXt3-4-32-vanilla$^{\ddagger}$\\ \midrule
2019 & 1st  & Tucodec& \cite{Zhou2019}& \checkmark & \checkmark& \checkmark  & \checkmark & $\times$& $\times$& \checkmark & Tuocodec\\
2019 & 2nd& ETRI$^{\dagger}$ & \cite{Cho2019a}& $\times$& \checkmark& $\times$& \checkmark & \checkmark & \checkmark & $\times$& $\sim$ Backbone + GRDN$^{\nabla}$\\
2019 & 2nd  & Joint  & \cite{ZhouJ2019}  & \checkmark & \checkmark& \checkmark  & $\times$& $\times$& $\times$& $\times$& Tucodec (no RAB)\\
2019 & 3rd  & NJUVisionPSNR & \cite{Lu2019}& $\times$& \checkmark& \checkmark  & $\times$& $\times$& $\times$& $\times$& $\sim$  Tucodec (no RAB)  \\
2019 & 4th  & Vimicro& \cite{Li2019}& \checkmark & \checkmark& \checkmark  & $\times$& $\times$& $\times$& $\times$& Tucodec (no RAB)\\ \bottomrule
\end{tabular}
}
    \caption{\label{tab:CLIC_res}Successful CLIC entries and their precursors for 2018 and 2019. CLIC uses multiple compression metrics (PSNR, MS-SIMM etc) and while success on one metric generally implies success on another, this is not always the case. When there is ambiguity, we have given the highest positioning a system achieved. In the final column we have provided the name of the system according to the naming conventions we define in Section~\ref{sec:contrib_architecture} ---we implemented and evaluated all of the named variants. In this column a `$\sim$' symbol implies our implemented system is similar, but not equivalent, to the model in question.\\
    \tiny \textsuperscript{*} To our knowledge, these teams did not produce a publication detailing their approach.\\
    $^\dagger$ This system was also the fastest in 2019.\\
    $^\mathsection$The original model did not have skip-connections but it otherwise identical to our `ResNeXt3-3-N-vanilla'.\\
    $^\ddagger$The authors use `wide-activated residual blocks' in which the channel size is increased by a factor of 4 internally within the block. The ResNeXt authors \cite{Xie2017a} have shown that this is equivalent to their system with a cardinality of 4. \\
    $^\nabla$This system was not an end-to-end AE but used the VVC compression standard \cite{VVC2019}. 
    }
\end{table} 

In lossy Image Compression (IC), an image is contracted to a more space-efficient representation with a small loss of information. Since 2017, CAE-based IC systems have started to outperform traditional IC algorithms \cite{Theis2017} such as JPEG and JPEG 2000\footnote{In fact, a new \textit{lossless} CAE-based IC format L3C was proposed in May 2019 that also outperforms traditional lossless methods \cite{Mentzer2019}.}. There are a few key differences between the IC problem and ours\footnote{Specifically, that the encoder output in an IC system is a bitstream while in our case it is simply a vector of floats. This adds complexity in comparison with our system in a number of ways: firstly, there is a tradeoff in IC between the compressed size and quality so most systems are trained with a multi-task loss of bitstream entropy and reconstruction error; secondly, the quantization operation is non-differentiable and therefore the systems cannot be trained by backpropagation directly; and finally these systems require `importance maps' \cite{Li2018}, a form of attention to determine how many bits should be allocated to each region of the image.} but the similarities mean that the IC-CAE literature provides a useful starting point. Crucially, unlike physical-field compression, the IC problem has received a great deal of attention from the machine learning community. For example, CLIC or `Challenges on Learned Image Compression' runs annually during CVPR to find the state of the art in lossy IC. We used the papers of the CLIC winners and runners up for the years 2018 \cite{Zhou2018,Chen2018,Fan2018} and 2019 \cite{Zhou2019,Fan2018,Cho2019a,ZhouJ2019,Lu2019,Li2019} and followed their citations to give approximately 40 relevant machine learning papers. The successful entrants used a variety of architectural components including attention-based models \cite{Bahdanau2014}, GDNs (Generalised Divisive Normalisation transformations) \cite{Balle2015}, multi-scale learning \cite{Ronneberger2015} and complex residual blocks (RBs) \cite{He2016,Xie2017a,Zhang2020,Woo2018,Zhang2019} in order to improve their compression quality. A summary of these CAE architectures and an overview of their constituent elements is given in Table~\ref{tab:CLIC_res}. We have implemented a variant on all of these systems according to a framework described in Section~\ref{sec:contrib_architecture}. 

Of particular note, are the \textit{Tucodec} team who came first in CLIC-2018 and CLIC-2019. In 2019, two of the top-five finishing teams (\textit{Joint} \cite{ZhouJ2019} and \textit{Vimicro} \cite{Li2018}) used a network that was virtually identical to \textit{Tucodec}'s entry from the previous year. However, the Tucodec team improved on their previous design with the addition of Residual Attention Blocks (RABs)\footnote{We note that the \textit{Tucodec} authors assert incorrectly throughout their paper that they are using \textit{non-local} attention blocks or `RNABs' instead of the simpler local-attention RABs (proposed in the same paper by \cite{Zhang2019}). We have used RABs in our implementation of their system.} and won the competition for a second year running. In our experiments in Section~\ref{sec:expt_architecture} we found that our 3D extension of their 2019 system (see \fref{fig:zhou2019}) performed consistently better than the other architectures that we considered. 

\begin{figure}[!htb]
    \center{\includegraphics[width=.7\textwidth]{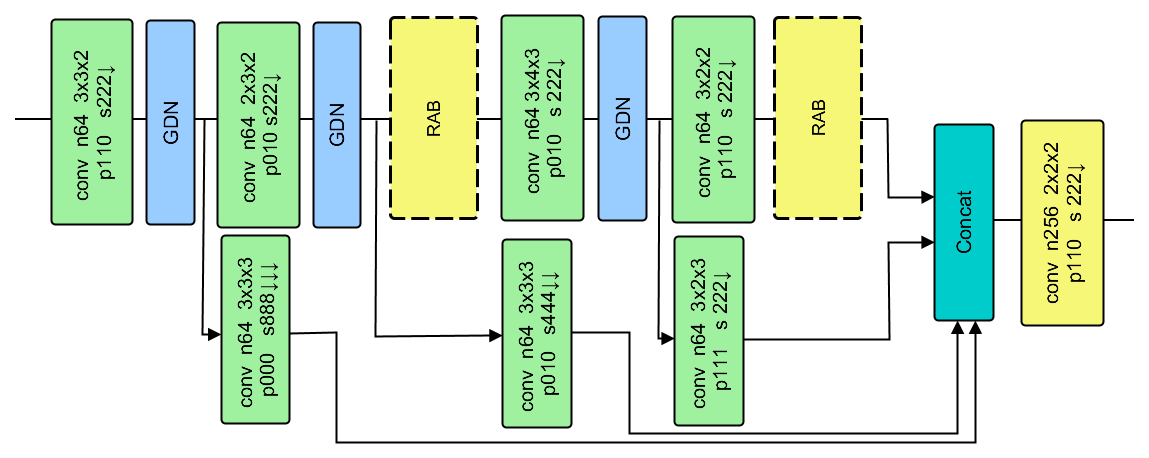}}
    \caption{\label{fig:zhou2019}The \textit{Tucodec} encoder \cite{Zhou2019}. The decoder is symmetric except that it does not have the multi-scale path. All parameters are specific to our 3D implementation. }
\end{figure}
      
%%%%%%%%%%%%%%%%%%%%%%%%%%%%%%%%%%%%

\section{`Bi-reduced space' formulation of DA} \label{sec:contrib_theory}
In this section we describe our proposed DA formulation with AEs and discus its computational cost. The new approach involves non-trivial changes to the incremental CVT formulation in (\ref{eq:3d_preconditioned}). A subtlety to note is that there are two `reduced' spaces in this case: the reduced space of size $S$ introduced by the CVT and the reduced space of size $m$ introduced by the encoder-decoder framework. As our method utilises both of these spaces we refer to it as the `bi-reduced space' formulation. To avoid confusion, we describe the AE space as the `latent space' (denoted with $\mathbf{z}$) and the CVT space as the `reduced space' (denoted with $\mathbf{w}$).

In Section \ref{sec:def_const}, we define a number of quantities required for our formulation in \ref{sec:proposed_formulation} and introduce a series of constraints and assumptions under-which we can show (in Section \ref{sec:equivalence}) that the proposed approach is equivalent to the Parish et al. CVT formulation \cite{Parrish1992}.

\subsection{Definitions and constraints} \label{sec:def_const}

\subsubsection{Constraint: mean-centred data}
All data is mean-centred with respect to the historical mean $\bar{\bm{x}}^b$. As this sets $\bm{x}^b = \bm{0}$ for all equations in Section~\ref{sec:background},  (\ref{eq:AE_def}) has multiple equivalent expressions:
\begin{align}
    \begin{split}
         f(\delta\bm{x}_i) &= f(\bm{x}_i - \bm{x}^b) = f(\bm{x}_i) = \mathbf{z}_i \\
    g(\delta\bm{z}_i) &= g(\mathbf{z}_i) = \bm{x}_i - \bm{x}^b= \bm{x}_i  \\
    \end{split}
\end{align}
This constraint is necessary for Section~\ref{sec:equivalence} but has the added benefit of ensuring that encoder inputs are symmetrically distributed about $\bm{0}$. The absolute state can be reconstructed by adding $\bar{\bm{x}}^b$ to the mean-centred state. Whenever the DA performance is evaluated in this paper the absolute state is used. 

\subsubsection{Definitions}
We define a series of quantities below where a subscript $l$ implies the matrix or vector has been replaced by its latent-space equivalent:
\begin{itemize}
    \item Let $f^o(\cdot)$ be the `observation encoder' which maps from the full observation space of size $M$ to the reduced observation space of size $M_l$ such that:
\begin{align}
    f^o(\cdot) & \coloneqq \bm{H}_l f  \bm{H}^+ (\cdot) \in \mathbb{R}^{M_l \times M} \\
    f^o(\cdot) &= \ \ f  \bm{H}^+ (\cdot)  \ \ \ \in \mathbb{R}^{m \times M}  \label{eq:foReduced}
\end{align}
where $\bm{H}^+ \in \mathbb{R}^{n \times M}$ is the generalised inverse of $\bm{H}$ which maps from the observation space to the full space and $\bm{H}_l f\bm{H}^+(\cdot)$ is the compound operator where the three transforms are applied in the order $\bm{H}^+, f, \bm{H}_l$ to an observation $ \in \mathbb{R}^{M} $. This is an under-determined problem so there will be many equivalent $\bm{H}^+$ operators. $\bm{H}_l  \in \mathbb{R}^{M_l \times m}$ is the latent space observation operator which maps from the latent space of size $m$ to the latent observation space of size $M_l$. We have implicitly defined $\bm{H}_l \coloneqq \bm{I}$ which implies $M_l = m$. Another way of thinking of this is that the entire latent space is observable to us.
\item Let $\bm{V}_l$ be the latent space equivalent of $\bm{V}$ such that:
\begin{align}\label{eq:V_l}
\begin{split}
    \bm{V}_l = f(\bm{V}) &= \big[f(\delta \bm{x}_0^b),  \ f(\delta \bm{x}_1^b),  \ ...,  \ f(\delta \bm{x}_S^b)\big]  \\
   \ \ \ \bm{V}_l &= \big[\mathbf{z}_0^b,  \ \mathbf{z}_1^b,  \ ...,  \ \mathbf{z}_S^b\big] \in \mathbb{R}^{ m \times S} \\
\end{split}
\end{align}
Note that we are representing the information in matrix $\bm{V}$ of size  $ \mathbb{R}^{ n \times S}$ in a matrix of size  $\mathbb{R}^{ m \times S}$. In our implementation $m \sim 0.001n$ so this is an $\mathcal{O}(10^3)$ reduction in the size of our background covariance representation.
\item Let $\bm{d}_l$ be the latent misfit such that:
\begin{equation} \label{eq:latent_d}
    \bm{d}_l \coloneqq  f^o (\bm{d})  \in \mathbb{R}^{ m }
\end{equation}
\item Let $\bm{R}_l$ be the latent observation error covariance. Recall that the full-space covariance $\bm{R}$ is computed over the observation error $\bm{\epsilon} = \bm{H}\bm{x}^b - \bm{y}$ and as such, for observations $\bm{y}$, can be calculated as $\bm{R} = \mathbb{E}[ \ \bm{\epsilon} \bm{\epsilon}^T \ ]$. We define an equivalent observation covariance matrix $\bm{R}$ such that:
\begin{align}
    \bm{R}_l &\coloneqq \mathbb{E}[ \ \bm{\epsilon}_l \bm{\epsilon}_l^T \ ] \label{eq:R_l_1}\\
    \text{where      } \ \ \ \bm{\epsilon}_l &= f^o (\bm{\epsilon})
\end{align}

\end{itemize}

\subsubsection{Assumption: Orthonormal latent features}
If the full-space observation errors $\bm{\epsilon}$ are uncorrelated (and $\bm{R} = \sigma_0^2 \bm{I}$), then the latent observation errors $\bm{\epsilon}_l$ will also be uncorrelated as they are derived directly from these full space errors giving:
\begin{align}
     \bm{R}_l &= \sigma_l^2 \bm{I} \label{eq:R_l_2}
\end{align}
where $\sigma_l$ is the latent observation error standard deviation. However, we note that the uncorrelated observation assumption breaks down as the observation locations move closer to one another (i.e.\ as $M$ increases). In this work we consider the case in which $M=n$, pushing this traditional DA assumption to breaking point. Nevertheless, (\ref{eq:R_l_2}) will hold in all scenarios if we impose the additional constraint that all latent features are orthonormal to one another. It is worth making a brief comment on the conditions under which this will be true:
\begin{enumerate}
    \item Latent dimensions are not orthogonal in the general case but they will be \linebreak `approximately' perpendicular. To see why intuitively: if an AE is producing a good reconstruction, the latent hyperplane must span the majority of the data distribution manifold. If it is able to do this with a very small number of latent features then it must be rare for latent features to double-up and span identical areas of the data distribution. 
    \item It is possible to enforce latent orthogonal features with the use of a Variational Autoencoder \cite{Doersch2016}. This approach is not explored in this paper for reasons discussed in Section~\ref{sec:architecture_search} but this is a clear avenue for future work. 
    \item To obtain unit length features, batch normalisation could be used in the encoder. However, like Chen et al. \cite{Chen2018}, we found that batch normalisation greatly hampered the AEs ability to produce good reconstructions and hence we did not use it in the backbone network of our models. For this reason, we do not make statements on the relation between the magnitude of $\sigma_0$ and $\sigma_l$ (although in practice we treat them as equal). 
\end{enumerate}

\subsubsection{Observation encoder in practice}
In the general case, calculation of the latent misfit $\mathbf{d}_l$ in (\ref{eq:latent_d}) requires evaluation of the observation encoder operator $f^o$ on the full-observation space:
\begin{align}
    \bm{d}_l &=  f^o (\bm{d})\\
    \bm{d}_l &= f^o  (\bm{y} - \bm{H} \bm{x}^b)
\end{align}
where as a result of the mean-centering constraint we have:
\begin{equation}\label{eq:dl_final}
    \bm{d}_l = f^o (\bm{y} )   
\end{equation}
We hypothesise that $f^o(\bm{y})$ could be modelled with a CAE framework that adds an implicit interpolation network $\bm{H}^+$ to the trained encoder network $f (\cdot)$ as shown in \fref{fig:f^o_scheme}. This system could be trained with the reconstruction error over $\bm{y}$ and $\hat{\bm{y}} = \bm{H}f^o(\bm{y})$. 

\begin{figure}[!htb]
        \center{\includegraphics[width=0.6\textwidth]
        {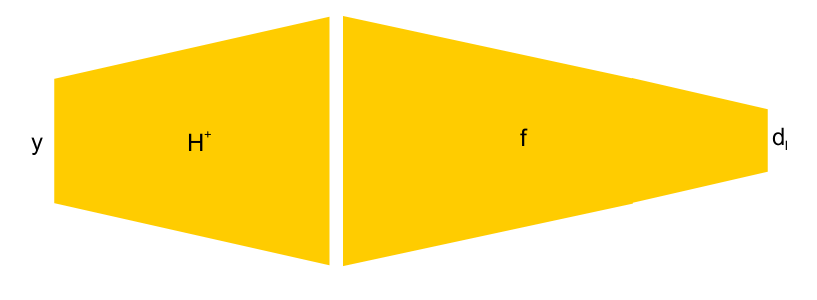}}
        \caption{\label{fig:f^o_scheme}Scheme to create the $f^o (\cdot)$ operator. This network is trained end-to-end with the observation reconstruction error. $\bm{H}^+$ is a convolutional network and $f(\cdot)$'s weights are fine-tuned to create $f^o (\cdot) = f\bm{H}^+(\cdot)$.}
      \end{figure}

However, in this work we are using synthetic data meaning that the full state is available as observation. As such, we sidestep this complexity and use: 
\begin{equation}
    \bm{H} = \bm{H}^+ = \bm{I}
\end{equation} 
which implies: 
\begin{align}
    \begin{split}
        \bm{d}_l &=  f\bm{H}^+(\bm{y})\\
        \bm{d}_l &\equiv  f(\bm{y}) \equiv f(\bm{x}^{obs}) \equiv \bm{z}^{obs} \\
        \text{where} \ \ \ \bm{y} &= \bm{x}^{obs}\in \mathbb{R}^n
    \end{split}
\end{align}
In other words, we use the encoder to obtain our latent misfit. Verifying the feasibility of the scheme in \fref{fig:f^o_scheme} is a necessary criterion for using the proposed approach operationally. 

\subsection{Proposed 3D-VarDA formulation}\label{sec:proposed_formulation}
Our proposed bi-reduced space formulation is:
\begin{align} \label{eq:3d_latent}
    \begin{split}
        &\mathbf{w}^{DA}_l = \argmin_{\mathbf{w}_l} J(\mathbf{w}_l) \\
    J(\mathbf{w}_l)  = \frac{1}{2} &\mathbf{w}_l^T\mathbf{w}_l + \frac{1}{2}\norm{\bm{d}_l - \bm{V}_l\mathbf{w}_l}^2_{\bm{R}_l^{-1}}
    \end{split}
\end{align}
Once this has been minimised in the reduced space the result can be restored to the full space in a two-stage transformation:
\begin{enumerate}
    \item Multiplication by $\bm{V_l}$ to move from the reduced space representation $\in \mathbb{R}^S$ to the latent space $\in \mathbb{R}^m$.
    \item Applying the decoder $g(\cdot)$ to move from the latent space to the full space $\in \mathbb{R}^n$.
\end{enumerate}

 Overall this gives: 
\begin{equation}
    \delta \bm{x}^{DA} = g(\bm{V_l}\mathbf{w}_l^{DA})
\end{equation}

Note that, allowing for the different definitions of the latent variables $\bm{V}_l,  \ \bm{d}_l$ and $\bm{R}_l$,\linebreak the formulation in (\ref{eq:3d_latent}) is identical to the mono-reduced space formulation in (\ref{eq:3d_preconditioned}). As such, we were able to use the same cost-function, gradient and minimisation implementations when comparing the two approaches\footnote{This gave us confidence that the execution time comparisons in Section~\ref{sec:expt_TSVDvsAE} are not biased towards either method as the result of implementation details.}.

%%%%%%%%%%

\subsection{Proof of equivalence} \label{sec:equivalence}
Our proposed formulation is equivalent to the mono-space formulation in (\ref{eq:3d_preconditioned}) in that:
\begin{equation} \label{eq:equiv}
    \mathbf{w}^{DA} = \mathbf{w}_l^{DA}
\end{equation} 
This is true under three assumptions:
\begin{enumerate}
    \item For high-performing autoencoder, we assume that the AE compression is lossless meaning $g(f(\bm{x})) =  \bm{x}$. 
    \item All features in the latent representation $\mathbf{z}$ are orthonormal as discussed above.
    \item The \textit{observation} space contains sufficient information to construct a good approximation of the full space $\bm{x}$. This is a necessary condition in the creation of the $f^o$ operator in \fref{fig:f^o_scheme} and is more likely to hold when $M$ is large. This assumption is discussed in more detail in the proof of Lemma 5 below.
\end{enumerate}

We will state and prove a series of lemmas to produce the result in (\ref{eq:equiv}). 
\begin{lemma}
%The exact solution of %$\mathbf{w}^{DA} = (\bm{I} + \bm{V}^T \bm{H}^T \bm{R}^{-1} \bm{H} \bm{V})^{-1} \bm{V}^T \bm{H}^T \bm{R}^{-1}\bm{d}$
%\end{lemma}
%\begin{lemma}
Let $\mathbf{w}^{DA}$ and $\mathbf{w}^{DA}_l$ denote the solutions of \eqref{eq:3d_preconditioned} and \eqref{eq:3d_latent} respectively, we have that 

\begin{equation}
\mathbf{w}^{DA} = (\bm{I} + \bm{V}^T \bm{H}^T \bm{R}^{-1} \bm{H} \bm{V})^{-1} \bm{V}^T \bm{H}^T \bm{R}^{-1}\bm{d}
\end{equation}
and 
\begin{equation}
    \mathbf{w}^{DA}_l = (\bm{I} + \bm{V}_l^T  \bm{R}_l^{-1}  \bm{V}_l)^{-1} \bm{V}_l^T \bm{R}_l^{-1}\bm{d}_l
    \end{equation}

{\bf Proof}: 
%\textbf{Proof of Lemmas 1 and 2}:
%the proofs of these Lemmas are identical as the cost functions are mathematically equivalent where each operator is replaced by its latent equivalent in the bi-reduced space formulation and $\bm{H}_l = \bm{I}$.
The gradient of (\ref{eq:3d_preconditioned}) is:
\begin{align*}
    \nabla J(\mathbf{w})  &= 
     \mathbf{w} - \bm{V}^T\bm{H}^T\bm{R}^{-1} \big(\bm{d} - \bm{H}\bm{V}\mathbf{w} \big)
\end{align*}
Setting this to zero and solving for $\mathbf{w}$ will give the optimal value $\mathbf{w}^{DA}$ as required to complete the proof: 
\begin{align*}
     \bm{V}^T\bm{H}^T\bm{R}^{-1} \bm{d}  &=
     \big( \bm{I} + \bm{V}^T \bm{H}^T \bm{R}^{-1} \bm{H} \bm{V}\big)\mathbf{w}^{DA}  \\
     \mathbf{w}^{DA} &= (\bm{I} + \bm{V}^T \bm{H}^T \bm{R}^{-1} \bm{H} \bm{V})^{-1} \bm{V}^T \bm{H}^T \bm{R}^{-1}\bm{d}
\end{align*}
%\end{proof}
\end{lemma}

As the cost functions in \eqref{eq:3d_preconditioned} and \eqref{eq:3d_latent}  are mathematically equivalent where each operator is replaced by its latent equivalent in the bi-reduced space formulation and $\bm{H}_l = \bm{I}$, we can write the exact solutions of $\mathbf{w}^{DA}$ and $\mathbf{w}^{DA}_l$ in the following form:
\begin{align*}
    \mathbf{w}^{DA} &= (\bm{I} + \bm{A})^{-1} \bm{b} \\
    \mathbf{w}_l^{DA} &= (\bm{I} + \bm{A}_l)^{-1} \bm{b}_l 
       \end{align*}
    where, the matrices $\bm{A}$ and $\bm{A}_l$ are such that:
    \begin{equation}\label{eq:matrixA}
        \bm{A} = \bm{V}^T \bm{H}^T \bm{R}^{-1} \bm{H} \bm{V},
    \end{equation}
    \begin{equation}\label{eq:matrixAl}
        \bm{A}_l = \bm{V}_l^T  \bm{R}_l^{-1}  \bm{V}_l
    \end{equation}
    and $\bm{b}$ and $\bm{b}_l$ are:
    \begin{equation}\label{eq:matrixb}
         \bm{b} = \bm{V}^T \bm{H}^T \bm{R}^{-1}\bm{d}
    \end{equation}
    \begin{equation}\label{eq:matrixbl}
        \bm{b}_l = \bm{V}_l^T \bm{R}_l^{-1}\bm{d}_l
    \end{equation}
    %\text{where} \ \ \ \ \ \ \ \ \ \ \ \bm{A} &= \bm{V}^T \bm{H}^T \bm{R}^{-1} \bm{H} \bm{V} \\
   % \bm{A}_l &= \bm{V}_l^T  \bm{R}_l^{-1}  \bm{V}_l \\
  %  \bm{b} &= \bm{V}^T \bm{H}^T \bm{R}^{-1}\bm{d} \\
 %   \bm{b}_l &= \bm{V}_l^T \bm{R}_l^{-1}\bm{d}_l
%\end{align*}
\begin{lemma}\label{lemma:AandAl}
Let $\bm{A}$ and $\bm{A}_l$ be the matrices defined in \eqref{eq:matrixA} and in \eqref{eq:matrixAl} respectively. The following result held:
    \begin{equation}\label{eq:A_l_A_lemma}
        \bm{A} =\bm{A}_l
    \end{equation}
\end{lemma}

\begin{lemma}\label{lemma:bandbl}
Let $\bm{b}$ and $\bm{b}_l$ be the matrices defined in \eqref{eq:matrixb} and in \eqref{eq:matrixbl} respectively. The following result held:
    \begin{equation}\label{eq:b_l_b_lemma}
         \bm{b} =\bm{b}_l
    \end{equation}
\end{lemma}
The overall result in equation~\eqref{eq:equiv} follows directly from Lemmas~\ref{lemma:AandAl} and \ref{lemma:bandbl} but to prove these we need two further results in Lemmas~\ref{lemma:VlV} and \ref{lemma:RlR}.
\begin{lemma}\label{lemma:VlV}
Let $\bm{V}_l$ be the reduced matrix as defined in \eqref{eq:V_l}, $f^o$ defined in \eqref{eq:foReduced}, $\bm{H}$ as in \eqref{eq:misfitH} and $\bm{V}$ as in \eqref{eq:V_full}. The followind result held:
    \begin{equation}\label{eq:V_l_V_lemma}
        \bm{V}_l = f^o \bm{H} \bm{V} 
    \end{equation}
    \begin{proof}
    \textbf{Proof of Lemma~\ref{lemma:VlV}}:
    by definition, we have 
    \begin{equation}\label{eq:VlVinLemma}
        \bm{V}_l \coloneqq f \bm{V}
    \end{equation}
    then, from \eqref{eq:VlVinLemma} we have
    \begin{equation}
         \bm{V}_l \approx f \bm{H}^+ \bm{H} \bm{V} 
    \end{equation}
    Then the thesis follows from the definition of $f^o$ in \eqref{eq:foReduced}:
    \begin{align*}
    \begin{split} \label{eq:Vl_expression}
       % \bm{V}_l &\approx f \bm{H}^+ \bm{H} \bm{V} \\
        \bm{V}_l &= f^o \bm{H} \bm{V} .
    \end{split}
    \end{align*}
   We note that $\bm{H}^+ \bm{H}$ acts is an information bottleneck operator in which only information contained in the observation locations is propagated from $\bm{V}$. %Note that $\bm{H}^+ \bm{H} \neq \bm{I}$ in general but 
    Here we can assume that $f \bm{V} \approx f \bm{H}^+ \bm{H} \bm{V}$ due to some weak assumptions:
    %is a much weaker result that suggests:
    \begin{itemize}
        \item The observation space of size $M$ contains sufficient information to construct a good approximation of the latent representation. This is equivalent to assumption iii) above because, if the observation space contains all information in the full space, by assumption i) it should also contain all information in the latent space. To show why this might be true consider that there must be large redundancies in the full space in order for the CAE framework to have any success. We argue in the following section that, in all practical scenarios, $m, S < M$. If a state of size $m$ can contain most of the information of a state of size $n$ it is not unlikely that a state of size $M$ might contain the same information. More concretely, in Section~\ref{sec:expt_TSVDvsAE} we demonstrate that the Arcucci et al. CVT with TSVD DA method~\cite{Arcucci2019a} suffers no degradation in accuracy when just 10\% of the total state space is used as observations ($M = 0.1n$) and there is only a 5\% degradation when $M = 0.01n$. 
        \item The reduced space of size $S$ also contains sufficient information to construct the latent representation of size $m$. This condition is implied by the lossless compression assumption i) as the reconstruction of the full state passes through the reduced space and the latent space. 
    \end{itemize}
    \end{proof}
\end{lemma}

\begin{lemma}\label{lemma:RlR}
Let $\bm{R}_l$ be the reduced matrix as defined in \eqref{eq:R_l_2}, $f^o$ defined in \eqref{eq:foReduced}, $\bm{H}$ as in \eqref{eq:misfitH} and $\bm{R}$ as in \eqref{eq:R_def}. The following result held:
   \begin{equation}\label{eq:R_l_R_lemma}
        \bm{R}_l^{-1} = ((f^o)^T)^{-1} \bm{R}^{-1} (f^o)^{-1} 
  \end{equation}
        \begin{proof}
        \textbf{Proof of Lemma~\ref{lemma:RlR}}:
     \begin{align*}
    \begin{split} \label{eq:Rl_expression}
    \bm{R}_l &\coloneqq \mathbb{E}[ \ \bm{\epsilon}_l \bm{\epsilon}_l^T \ ]\\
    \bm{R}_l &= \mathbb{E}[ f^o  \bm{\epsilon}  \bm{\epsilon}^T (f^o)^T   ]\\
    \end{split}
    \end{align*}
    because the observations are uncorrelated, we have:
     \begin{align*}
    \begin{split}
    \bm{R}_l &= f^o \ \mathbb{E}[ \ \bm{\epsilon}  \bm{\epsilon}^T   \ ] \ (f^o)^T\\
    \bm{R}_l &= f^o \ \bm{R} \ (f^o)^T\\
    \bm{R}_l^{-1} &= ((f^o)^T)^{-1} \bm{R}^{-1} (f^o)^{-1}
    \end{split}
    \end{align*}
    % where the third line follows because the observations are uncorrelated. 
    \end{proof}

%%%%%%%%%%%%%
    \begin{proof}
    \textbf{Proof of Lemma~\ref{lemma:AandAl}}: we use Lemmas 5 and 6 to give:
    \begin{align}
        \begin{split}
            \bm{A}_l &= \bm{V}_l^T  \bm{R}_l^{-1}  \bm{V}_l \\
             \end{split}
    \end{align}  
    from \eqref{eq:R_l_R_lemma} and \eqref{eq:V_l_V_lemma}, we have:
      \begin{align}
        \begin{split}
            \bm{A}_l &= \bm{V}^T \bm{H}^T (f^o)^T    ((f^o)^T)^{-1} \bm{R}^{-1} (f^o)^{-1} f^o \bm{H} \bm{V}  
             \end{split}
    \end{align}
    which gives the \eqref{eq:A_l_A_lemma}.
    \end{proof}
    \begin{proof}
    \textbf{Proof of Lemma~\ref{lemma:bandbl}}:
    \begin{align}
        \begin{split}
            \bm{b}_l &= \bm{V_l}^T \bm{R_l}^{-1}\bm{d_l} \\
             \end{split}
    \end{align}
             from \eqref{eq:R_l_R_lemma} and \eqref{eq:V_l_V_lemma}, we have:
             \begin{align}
        \begin{split}
            \bm{b}_l &= \bm{V}^T \bm{H}^T (f^o)^T ((f^o)^T)^{-1} \bm{R}^{-1} (f^o)^{-1} f^o(\bm{d})\\
           % \bm{b}_l &= \bm{b}
        \end{split}
    \end{align}
    which gives the \eqref{eq:b_l_b_lemma}.
    \end{proof}
\end{lemma}
This completes the proof that $\mathbf{w}^{DA} =   \mathbf{w}_l^{DA}$.

\subsection{Advantages over TSVD: Theory} \label{sec:background_tsvdvsAE}
Eigenanalysis techniques such as PCA and TSVD are alternative methods of producing reduced space representations of data and, as discussed, have been used canonically in preconditioned 3D-VarDA \cite{Arcucci2019a}. Having summarised the key components of our proposed system, it is now possible to discus the theoretical reasons why our method produces a) higher quality compression and b) is faster than the traditional methods. We verify these advantages experimentally in Section~\ref{sec:expt_TSVDvsAE}.
 
\subsubsection{Compression quality}
A well-trained CAE will produce higher quality reconstructions than those using TSVD for a number of reasons:
\begin{enumerate}
    \item The mean of the training data distribution can be stored `for free' in the decoder leaving space in the latent representation to encode sample variations.
    \item CAEs explicitly use location data and can therefore utilise properties like local smoothness in order to compress the input more efficiently\footnote{Note that DA localisation approaches do utilise location information but not in a way that increases compression quality \cite{Montmerle2018}.}.
    \item The latent features are created from non-linear combinations of the inputs meaning they are likely to be of greater expressive quality. This eigenanalysis approach is only optimal when the data is drawn from a Gaussian distribution\footnote{More formally, PCA truncated at mode $\tau$ gives optimal reconstruction for all \textit{linear} models of rank $\tau$.}.
    \item By design, in \textit{truncated} SVD, some of the information is intentionally discarded. This is not the case in the CAE framework. 
\end{enumerate}

\subsubsection{Computational Complexity}
The proposed method also has lower computational complexity in an online setting. Here `online costs' refers to any calculation that must take place when a new set of observations are made. `Offline costs' are everything else and includes the TSVD computation and the CAE training. The online complexities of the Parish et al. reduced space approach $R_{\text{on}}$ \cite{Parrish1992} and our bi-reduced space approach $B_{\text{on}}$ are:
\begin{align}
    R_{\text{on}} &=\mathcal{O}(I_1M^2 + nS) \\
    B_{\text{on}} &= \mathcal{O}(nm)
\end{align}
where $I_1$ is the number of iterations in the reduced space VarDA minimisation routine, $M$ is the number of observations, $S$ is the reduced space size (which is equal to the size of the historical data sample) and $m$ is the latent dimension size in our proposed method. We note that to achieve comparable accuracy with the two methods we will typically have $S > m,$ (and $M > S$) so $R_{\text{on}} > B_{\text{on}}$. We derive these results in the following sections. 

We note that our proposed method's online complexity $ B_{\text{on}}$ is independent of the number of observations $M$ meaning it is never necessary to arbitrarily reduce the number of assimilated observations in order to meet the latency requirements of the system. We also note that the cost of training a CAE is \textit{considerably} larger than of performing TSVD but as these operations occur offline they are not of primary importance in the creation of an operational system. We give the derivation of the online and offline complexities in the following sections but first discuss the encoder and decoder inference complexities: \\

\textbf{$f(\bm{x})$ and $g(\mathbf{z})$ complexity} \\
With input of size $n$ and output of size $m$, the encoder and decoder inference complexities are of order $\mathcal{O}(nm)$.
\begin{proof}
\textbf{Justification}: \textit{in the simplest possible encoder consisting of a single fully connected layer, the complexity of mapping from the full to the latent space would be $\mathcal{O}(nm)$ exactly. The convolutional case is more complex and will be given by  $\mathcal{O}(nK)$ for a architecture-specific constant $K$ but we think that logical CAE design choices give $m \approx K$ since:
\begin{enumerate}
    \item $K$ will not be $<<m$ as this would mean the encoder was introducing the information bottleneck at a location other than the latent space. 
    \item Similarly, $K$ should not be $>> m$ as this would negate the computational benefit of using convolutions over a linear network.
\end{enumerate}
A symmetric argument gives the decoder complexity as $\mathcal{O}(nm)$. }
\end{proof}
\subsubsection{Online Computational Cost Derivation}\label{sec:complexity}

There are two steps that contribute to the online-cost:
\begin{enumerate}
    \item The evaluation of the cost function and its gradient during the minimisation. 
    \item Restoring the calculated $\mathbf{w}^{DA}$ to the full space.
\end{enumerate} 
In the following, we use the symbol $\cdot$ to indicate the operation under consideration. 
\begin{lemma}
    The online complexity of step i) in the Reduced space method is:
    $$\mathcal{O}(I_1(MS + M^2))$$ 
    where $I_1$ is the number of iterations in the minimisation routine.
\end{lemma}

\begin{proof}
\textbf{Proof of Lemma 7}: \textit{we repeat the cost function (\ref{eq:3d_preconditioned}) here for convenience:}
\begin{align} \label{eq:3d_preconditioned_appendix}
    \begin{split}
         \mathbf{w}^{DA} = \argmin_{\mathbf{w}} J(\bf{w}) \\
    J(\mathbf{w}) = 
     \frac{1}{2} \mathbf{w}^T\mathbf{w} + \frac{1}{2}\norm{\bm{d} - \bm{H}\bm{V}\mathbf{w}}^2_{\bm{R}^{-1}} 
    \end{split}
\end{align} 

\textit{A naive implementation of \ref{eq:3d_preconditioned_appendix} (and its derivative) would be dominated by the matrix multiplication $\bm{H} \cdot \bm{V}$ but, as $\bm{H}$ and $\bm{V}$ always appear together, this quantity can be pre-computed and the minimisation complexity is independent of $n$. A single iteration of the VarDA minimisation has complexity $\mathcal{O}(MS + M^2)$ where the first term originates from $\bm{HV} \cdot \mathbf{w}$ while the second is from $(\bm{d} - \bm{HV}\mathbf{w})^T \cdot (\bm{d} - \bm{HV}\mathbf{w})$ where we are assuming $\bm{R}$ is diagonal. With $I_1$ iterations this gives Lemma 7. }
\end{proof}

\begin{lemma}
The online complexity of step i) in our bi-reduced space method is: 
$$\mathcal{O}(I_2(mS + m^2))$$ 
where $I_2$ is the number of iterations in the minimisation routine.
    \begin{proof}
    \textbf{Proof of Lemma 8}: this argument is almost identical to Lemma 7 except that we replace $\bm{H} \bm{V} \in \mathbb{R}^{M \times S}$ with $\bm{V}_l \in \mathbb{R}^{m \times S}$ in the cost function. Altering these dimensions gives the required result.
    \end{proof}
\end{lemma}

%%%%%%%%
\begin{lemma}
Restoring the calculated $\mathbf{w}^{DA}$ to the full space in the mono-reduced space formulation is in $\mathcal{O}(nS)$.
    \begin{proof}
    \textbf{Proof of Lemma 9}: the product $ \bm{V} \cdot \mathbf{w}^{DA}$ is in $ \mathcal{O}(nS)$.
    \end{proof}
\end{lemma}

\begin{lemma}
Restoring $\mathbf{w}_l^{DA}$ to the full space in the bi-reduced space formulation is in $ \mathcal{O}(nS)$.
    \begin{proof}
    \textbf{Proof of Lemma 10}: this requires computing $\bm{V_l} \cdot \mathbf{w}_l^{DA}$ followed by $g \cdot \bm{V_l} \mathbf{w}_l^{DA}$ which has complexity $\mathcal{O}(mS + nm) = \mathcal{O}(nm)$.
    \end{proof}
\end{lemma}

This gives an an overall \textbf{reduced space online complexity} of:
\begin{equation}\label{eq:R_on}
    R_{\text{on}} =\mathcal{O}(I_1(MS + M^2) + nS)
\end{equation}
and a \textbf{bi-reduced space online complexity} of:
\begin{equation}\label{eq:AE_complexity}
    B_{\text{on}} = \mathcal{O}(I_2(mS + m^2) + nm)
\end{equation}

Hence, the key comparison between the online complexity of the two methods is down to the relative sizes of variables $I_1, I_2, M, n, m \text{ and } S$. We assert that in most practical cases:
\begin{equation}\label{eq:complexity}
    R_{\text{on}} > L_{\text{on}}
\end{equation}

As many of these are user-chosen parameters we cannot prove that (\ref{eq:complexity}) \textit{always} holds but we can make concrete arguments about their ranges in practical settings. \\
\textbf{Argument 1:} \textit{$I_1 > I_2 $ in the majority of cases.}
    \begin{proof}
    \textbf{Justification}: \textit{We expect the problem to be better conditioned in the bi-reduced space than in the reduced space for the same reasons that it is better conditioned in the reduced space in comparison with the full space \cite{Hansen2006}.}
    \end{proof}
\textbf{Argument 2:} \textit{$M > S$ in the vast majority of cases.}
    \begin{proof}
    \textbf{Justification}: \textit{The Met Office uses $M =0.01n = 10 ^7$ \cite{MetOffice2019a}. They employ a combination of KFs and VarDA approaches \cite{Lorenc2018} but in the VarDA scheme, it is implausible that they would use anything close to $S = 10^7$ as $\bm{V}$ would be a matrix of size $10^7 \cdot 10^9 = 10^{16}$. We believe the same constraints will hold in all practical scenarios even if the number of observations is relatively small.}
    \end{proof}
\textbf{Argument 3: }\textit{$S > m$ for useful systems.}
    \begin{proof}
    \textbf{Justification}: \textit{We found that a value of $m$ that was a factor of $\mathbf{x}2.5$ smaller than $S$ gave superior DA accuracy compared with the traditional method. }
    \end{proof} 
\textbf{Argument 4:} \textit{$B_{\text{on}} =\mathcal{O}(nS)$ in the vast majority of cases.}
    \begin{proof}
    \textbf{Justification}: \textit{According to Argument 3 we have:
    \begin{align*}
        B_{\text{on}} &= \mathcal{O}(I_2(mS + m^2) + nm) = \mathcal{O}((I_2S + n)m)
    \end{align*}
    We found $I_2 = \mathcal{O}(10)$ meaning $I_2S << n$ and hence  $B_{\text{on}} = \mathcal{O}(nm)$. Concretely, with our values of $m =288$ and $S=791$ we found that upwards of 97\% of the execution time of the bi-reduced DA procedure was restoring $\mathbf{w}^{DA}_l$ to the full space.}
    \end{proof}
\begin{proof}

\end{proof}
Combining \textbf{Arguments 2-4}, gives the stated results in the previous section:
\begin{align*}
    R_{\text{on}} &=\mathcal{O}(I_1M^2 + nS) \\
    B_{\text{on}} &= \mathcal{O}(nm)
\end{align*}
and since we have argued that $n < S$, this implies $R_{\text{on}}$ is strictly greater than $L_{\text{on}}$. When $M$ is large ($M \geq \mathcal{O}( 0.05n)$) we can go further than this: \\
\textbf{Argument 5:} \textit{When $M > \mathcal{O}( 0.05n)$, $R_{\text{on}} =\mathcal{O}(I_1M^2)$. } 
\begin{proof}
    \textbf{Justification}: \textit{Clearly the exact values here will vary from one problem to another but we found that with $M = (0.01n)$, steps i) and ii) had approximately equal execution time but when $M$ rose much above this, the minimisation term dominated as a result of the quadratic complexity.}
\end{proof} 

We investigate the negative effect on DA accuracy of using fewer observations in Section~\ref{sec:expt_TSVDvsAE}.

\subsubsection{Computational Complexity: Offline}
In comparison, the offline costs are much larger for the proposed method in comparison with the formulation with TSVD: \\
\textbf{Reduced space offline cost}
\begin{equation}\label{eq:red_space_cost}
    R_{\text{off}} = \mathcal{O}(nS(M + S))
\end{equation}
\textbf{Bi-reduced space offline cost}
\begin{equation}\label{eq:bi_space_cost}
    B_{\text{off}} = \mathcal{O}(nSEm)
\end{equation} 
for $E$ epochs of training. We prove these results below.\\

Typically, $Em > M, S$ so $R_{\text{off}} < B_{\text{off}}$. In practice we found that $R_{\text{off}}< < B_{\text{off}}$: training a model took something on the order of 10-20 hours on a GPU whilst TSVD required approximately 10 minutes on the CPU.

\begin{proof}
\textbf{Derivation of (\ref{eq:red_space_cost})}: \textit{calculating SVD for $\bm{V} \in \mathbb{R}^{n \times S}$ where $S < n$ has complexity of $\mathcal{O}(S^2n)$ \cite{SVD2013}. It is also necessary to precompute $\bm{H} \cdot \bm{V_{\tau}}$ which is in $\mathcal{O}(MnS)=$ giving the overall complexity in (\ref{eq:red_space_cost})}.
\end{proof}

\begin{proof}
\textbf{Derivation of (\ref{eq:bi_space_cost})}: \textit{training for $E$ epochs, with $S$ historical samples, requires $ES$ evaluations of the encoder/decoder giving a complexity of $\mathcal{O}(SEmn)$. We must also precompute $\bm{V}_l = f \cdot \bm{V}$ which is in $\mathcal{O}(nmS)$ but this is comparatively small so the offline cost is as in equation (\ref{eq:bi_space_cost}).}
\end{proof}

%%%%%%%%%%%%%%%%%%%%%%%%%%%%%%%%%%%%%%%%%%%%%%%%%%%%%
\section{Architecture Search} \label{sec:architecture_search}

In this section we describe our architecture search framework (Section~\ref{sec:contrib_architecture}) and training configuration (Section~\ref{sec:contrib_training}) before detailing the results of this search in Section~\ref{sec:expt_architecture}. We found that the proposed approach was only successful relative to the Parish et al. approach when state of the art CAE architectures were used. This point is worth highlighting as many Data Assimilation practitioners use very simple neural networks in their research \cite{Liu2019,Quilodran2019}.

\subsection{Search Framework}\label{sec:contrib_architecture} 
In our process of finding a good CAE architecture, we were concerned that the successes of 2D image compression architectures might not be transferable to 3D spatial inputs. We attempted to minimise this risk by designing a framework within which it was feasible to systematically search an architecture space that approximately encompassed the design of every 2018/19 top-5 CLIC finisher in Table~\ref{tab:CLIC_res}. In order to achieve this, it was necessary to make small alterations to some of the original systems. As such, although we found that the \textit{Tucodec} model was vastly superior to our implementations of the other systems, it is possible that our small design variations mean the quoted performances are not representative of the original designs. We were not unduly worried by this possibility as our aim in this process was to find systems that performed well in our domain rather than make exact comparisons between image compression networks.

Following best-practices from the literature and the results from our early-stage experiments, all CAE designs that we searched had the following features:
\begin{enumerate}
    \item Encoder down-sampling occurs via strided-convolution \cite{Springenberg2015} rather than max pooling \cite{Krizhevsky2012} so that the network can learn its own sub-sampling routine.
    \item Encoder and decoder layers are mirrored to reduce the design space size\footnote{The single exception to this is the \textit{Tucodec} architecture in which the multi-scale path is removed in the decoder.}. Strided convolutions in $f(\bm{x})$ are replaced with transposed convolutions in $g(\mathbf{z})$.
    \item Convolutions have kernel size $k=3$ or less to reduce the number of parameters and computational cost (as these scale with $k^3$ for a 3D feature-map~\cite{Szegedy2015,Simonyan2015}).
    \item Convolutional down-sampling parameters are generated by our \texttt{ConvScheduler} class that has priorities of:
    \begin{enumerate}
        \item Avoiding addition of padding in later encoder layers as these can introduce artefacts in the reconstructed state. 
        \item Avoiding creation of feature maps that are not centred on the input as these are much harder to reconstruct in the decoder. This is achieved by refusing stride, $s$, padding, $p$ and kernel size, $k$ combinations that, when acting on input of width $W$, result in rounding in the floor operation when calculating the output size $= \big \lfloor \frac{W - k + 2p}{s} \big \rfloor + 1$.
    \end{enumerate} 
    \item Batch Normalisation is not used in the down-sampling backbone of the CAE architectures because our preliminary investigation showed that it resulted in reconstructions of considerably poorer quality\protect\footnotemark. This finding was also made by Chen et al. \cite{Chen2018}. 
    \item Batch Normalisation \textit{is} used in the residual blocks of the network. This was necessary to prevent activation and gradient overflow in deep networks.
    \item Latent size is fixed at $m=288$ to enable like-with-like comparisons between CAEs. This value was a reduction in state size by approximately three orders of magnitude.
\end{enumerate}

\footnotetext{We hypothesise that this degradation in quality was likely the result of loosing batch-specific averages that are crucial to reconstruction. In our work this problem was likely exacerbated by the fact that, as a result of memory pressures created by 3D input data, we were using small batch sizes of just $16$ or $8$ meaning the batch statistics have high variance.} 

For all models we investigated:
\begin{enumerate}
    \item Three activation functions: ReLU \cite{Hinton2010}, GDN \cite{Balle2015} and PReLU activations \cite{He2015a}.
    \item Four RBs: vanilla \cite{He2016} and NeXt RBs \cite{Xie2017a} (see \fref{fig:RBs}) each with and without the lightweight CBAM module \cite{Woo2018} after the original RB. 
\end{enumerate}
\begin{figure}[!htb]
    \center{\includegraphics[width=\textwidth]{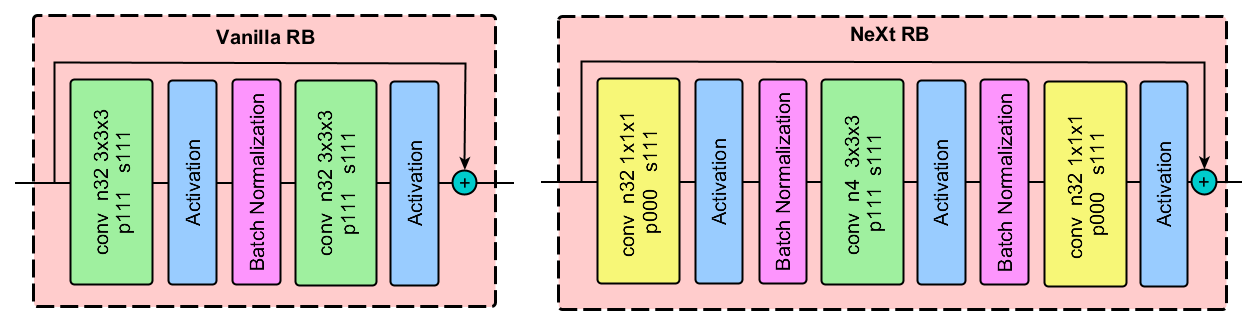}}
    \caption{\label{fig:RBs}The two basic RBs we evaluated. We also investigated the effect of placing CBAM blocks after each of these (see \ref{sec:appendix_architecture} for a diagram of the CBAM RB).}
\end{figure}
Note that vanilla RBs have considerably more parameters than NeXt RBs: for the 32 channel input versions shown in \fref{fig:RBs}, NeXt blocks have 2k parameters while vanilla blocks have almost 60k parameters. We found that we were able to reduce the \textit{Tucodec} decoder inference latency by $\mathbf{x}2.5$ by replacing vanilla RBs with NeXt RBs.

In designing a search strategy we noted that the models in Table \ref{tab:CLIC_res} fit into one of two categories:
\begin{enumerate}
    \item They are closely based on the \textit{Tucodec} 2018 entry.
    \item The encoder alternates between residual feature extraction and fully convolutional down-sampling operations.
\end{enumerate}

We found that we could capture most of the variation in the second category with the `backbone' encoder architecture shown in \fref{fig:backbone}. This design, without any `optional blocks' is a fully convolutional network with seven layers, the latter five of which down-sample the feature map. It is very similar to the encoder of Theis et al.~\cite{Theis2017} that was used in the first system which outperformed JPEG compression. The backbone is responsible for down-sampling while the optional blocks can introduce innovative feature extraction mechanisms. We require that any added blocks do not change the feature map size and are bypassed with a skip connection so that, at least in theory, they cannot hinder the down-sampling process. In Section~\ref{sec:expt_architecture} we evaluate four variants on this backbone that are summarised in Table~\ref{tab:backbone_vars}. For more information on these systems, including a discussion of our naming conventions, see \ref{sec:appendix_architecture}.

\begin{figure}[!htb]
    \center{\includegraphics[width=\textwidth]{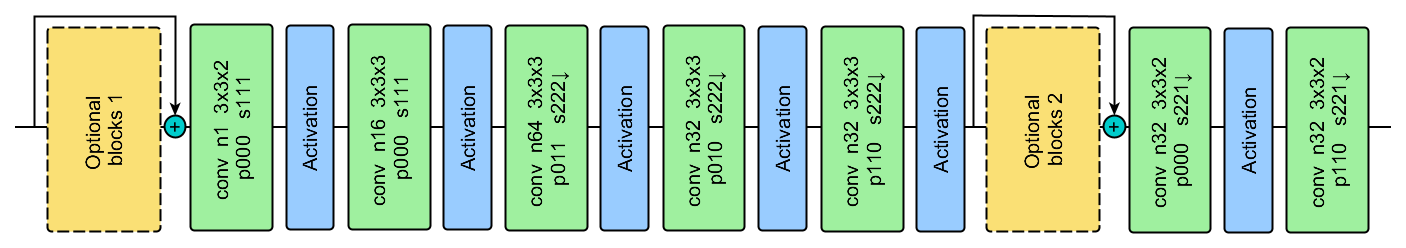}}
    \caption{\label{fig:backbone}Backbone encoder architecture. Convolutional parameters are specific to our input data with dimensions (C, Hx, Hy, Hz) = (1, 91, 85, 32) and latent dimensions of (32, 3, 3, 1). For `backbone' models, the optional blocks are empty. This architecture is loosely based on that in \cite{Theis2017}. The precise number of layers and convolutional parameters were chosen after an exploratory phase in which we compared a range of fully convolutional designs with between 5 and 11 layers.}
\end{figure}

\begin{table}[!htb]
    \let\center\empty
    \let\endcenter\relax
    \centering
    \resizebox{0.9\textwidth}{!}{\begin{tabular}{ccccc}
\toprule
\textbf{Model Name  }& \textbf{Optional Blocks 1 }& \textbf{Optional Blocks 2 } & \textbf{Figure} & \textbf{Reference}\\ \midrule
Backbone & -  & - & Fig. \ref{fig:backbone}& - \\
ResNeXt-L-N & - & L ResNeXt layers, cardinality N & Fig. \ref{fig:ResNeXtLayers}-\ref{fig:ResNeXt} & \cite{Xie2017a}\\
RAB-L & -  & L RABs  & Fig. \ref{fig:RAB}& \cite{Zhang2019}\\
GRDN  & 1 GRDN& -& Fig. \ref{fig:GRDN}&\cite{Zhang2020} \\ \bottomrule
\end{tabular}
}
    \caption{\label{tab:backbone_vars}Variants on the backbone encoder in \fref{fig:backbone} that are evaluated in Section~\ref{sec:expt_architecture}.}
\end{table} 

\subsection{Training Configuration}\label{sec:contrib_training}

\subsubsection{Data}\label{sec:data}
We used simulated data from a single run (988 time-steps) of the open-source, finite-element, fluid dynamic software Fluidity on a small domain in Elephant and Castle in South London. This system had 100,040 states spread over a region of size (x, y, z) = (700m, 650m, 250m). Fluidity uses an adaptive unstructured mesh in order to provide high resolution in regions of interest without requiring this same resolution in locations where there is little variation.
As convolutional kernels work on the assumption that adjacent states are equally spaced, it was necessary to interpolate between the points in the unstructured mesh to create a regular 3D grid with 247,520 evenly spaced points in the shape (91, 85, 32)\footnote{This increased the number of points by a factor of 2.5 as we found that a large amount of detail was lost in high-variance locations when the original number of points was used.}. When evaluating the reduced space approach, this interpolated input was flattened before use. The data was placed in time-step order and the first 80\% was used for the training set. All of the data was normalised using the training data statistics. In order to ensure a fair comparison with traditional TSVD methods, the reduced space $\bm{V}_{\tau}$ was calculated using the training set only.

\subsubsection{Augmentation}
With a view to increasing model generalisation, we investigated a number of regularisation techniques. We found that dropout harmed performance, even when it was applied channel-wise, and only to latter layers as recommended in \cite{Tompson2015}. Similarly, preliminary experiments showed that weight decay resulted in a small degradation in performance. As such, data augmentation was our only method of regularisation. We did not find any augmentation strategies for physical fields in the literature and decided that the only appropriate augmentation strategy was one borrowed from imaging: 3D `field-jitter' (the 3D mono-channel equivalent of colour-jitter). This involves injection of normal noise of amplitude $p\bm{\sigma}$ at $r$ of the state locations where $p, r \in [0, 1]$ and $\bm{\sigma}$ is the state standard deviation. With regards to other image augmentation strategies, we did not judge it appropriate to crop the inputs because the state space is of fixed size and hence there is no benefit in the CAE learning to compress inputs at variable sizes. Similarly, we did not flip the inputs horizontally as there are buildings in our domain which should stay fixed. 

\subsubsection{Training Duration}
Our Backbone architecture takes approximately 15 hours, and 400 epochs to converge on a NVIDIA Tesla K80. Some of the more heavyweight models take upwards of 40 hours. In order to conserve our resources and reduce the iteration time, we used a maximum of 150 epochs during our architecture search which amounted to an 6-15 hour period. Most systems had not converged by this point but we found that, for the sample of models that we trained to convergence, the performance ordering at convergence was almost completely unchanged from 150 epochs (see Figures \ref{fig:augmentation} and \ref{fig:fine_tuning}\footnote{Note that the training data-set metrics in these Figures are noisy but that the test-set values are stable and consistently ordered.}). All comparisons between AE architectures in Section~\ref{sec:expt_architecture} are made with this constraint and the quoted data assimilation figures in these sections should be used for comparative purposes only. We note that this may bias results towards smaller models that train more quickly but, all things considered, we would prefer a bias in this direction. We only trained a single model for each configuration.

\subsubsection{Evaluation metric}
When evaluating our systems' data assimilation performance we follow \cite{Arcucci2019a} in using the following quantity which we refer to as the `DA MSE':
\begin{equation}
    MSE(\bm{x}^{DA}) = \frac{\norm{\bm{x}^{DA} - \bm{x}^{obs}}_2}{\norm{\bm{x}^{obs}}_2}
\end{equation}
The equivalent quantity for the background state $MSE(\bm{x}^{b})$ is referred to as the `ref MSE' and is the value of the $MSE(\cdot)$ before DA has taken place. If a DA MSE is lower than the ref MSE, this implies the approach is performing better than the DA baseline system which always predicts the historical mean $\bm{x}^b$ regardless of the observations. The average ref MSE over the 197 test samples is 1.0001\footnote{The fact that this value is close to 1 is coincidental as we undo the normalisation before calculating this value.}. Unless otherwise stated, in all cases in which a single DA MSE is provided, we give the average over the test set.

\subsubsection{Hyperparameters} \label{sec:Implem_misc}
We trained our models with the Adam optimiser with the Pytorch default parameters of $\beta_1$= 0.9 and $\beta_2$ = 0.999. We used a fixed learning rate of 0.0002, He et al. initialization~\cite{He2015a}, and batch size of 16 for most models as this was largest multiple of eight at which the fp32 model, gradients and data could fit in 11GB of available GPU memory. One exception to this was the GDRN model which would only run at batch size 8.

\subsection{Architecture Search Results}\label{sec:expt_architecture}
\begin{table}[!htb]
    \let\center\empty
    \let\endcenter\relax
    \centering
    \resizebox{\textwidth}{!}{\begin{tabular}{lcccc}
\toprule
\multirow{2}{*}{ \textbf{Model} } & \multirow{2}{*}{ \textbf{Best DA MSE} } & \textbf{Relative Improvement} & \multirow{2}{*}{ \textbf{Best RB Type} } & \multirow{2}{*}{ \textbf{Best Activation} }   \\
 &   & \textbf{over Backbone}   &    &     \\ \midrule

Backbone   & 0.2309 & 0.00\%  & -  & PReLU   \\
ResNeXt-L-N    & 0.1900 & 17.71\% & Vanilla + CBAM    & GDN \\
RDB3-L-N   & 0.1865 & 19.21\% & Vanilla + CBAM    & GDN \\
RAB-L & 0.1917 & 16.98\%  & NeXt    & PReLU   \\
GRDN  & 0.1689 & 26.85\%    & NeXt + CBAM  & GDN  \\
Tucodec    & \textbf{0.0858} & \textbf{62.86\%} & vanilla  & PReLU  \\ %\midline
Ref MSE    & 1.0001 & - & -  & -   \\ 
\bottomrule
\end{tabular}
}
    \caption{\label{tab:architecture}A summary of the data assimilation performance of the best performing model variants \textbf{after 150 epochs of training} and the ref MSE for comparison.}
\end{table} 
Table~\ref{tab:architecture} gives a high-level summary of the results of our architecture search. Note that we also conducted experiments to find the best L and N values for the ResNeXt-L-N and RAB-L CAEs but presentation of these results is deferred to~\ref{sec:appendix_architecture}.

\subsubsection{Residual Block}\label{sec:expt_blocktype}
\begin{table}[!htb]
    \let\center\empty
    \let\endcenter\relax
    \centering
    \resizebox{0.9\textwidth}{!}{\begin{tabular}{lccccc}
\toprule
\multirow{2}{*}{ \textbf{Model} } & \multirow{2}{*}{ \textbf{Vanilla} } & \textbf{Vanilla}  & \multirow{2}{*}{ \textbf{NeXt} } & \textbf{NeXt}    & \textbf{Relative Improvement}   \\
 &    & \textbf{+ CBAM}   & &\textbf{ + CBAM  }& \textbf{over Backbone}   \\\midrule
ResNeXt3-27-4  & 0.2108    & 0.1998   & 0.2028 & \textbf{0.1907}  & 17.41\%     \\
RDB3-27-4    & 0.1950    & \textbf{0.1893}   & 0.1964 & 0.2005  & 18.02\%     \\
ResNeXt3-27-1  & 0.2031    & \textbf{0.1948 }  & 0.2106 & 0.2110  & 15.63\%     \\
RDB3-27-1    & 0.2064    & 0.2167   & \textbf{0.1968} & 0.2060  & 14.75\%     \\
ResNeXt3-3-8 & 0.2196    & \textbf{0.2051}   & 0.2174 & 0.2148  & 11.16\%     \\
RDB3-3-8     & \textbf{0.1958}    & 0.2125   & 0.2004 & 0.2013  & 15.20\%     \\
RAB-4   & 0.2277    & 0.1970   & \textbf{0.1917} & 0.1927  & 16.98\%     \\
GRDN    & 0.2297    & 0.2204   & 0.2300 & \textbf{0.1893}  & 18.00\%     \\
Tucodec & \textbf{0.0858}    & 0.3172   & 0.0890 & 0.1870  & \textbf{62.86\%}     \\
\bottomrule
\end{tabular}
}
    \caption{\label{tab:RBs}DA MSE variation with residual block in models trained for 150 epochs.}
\end{table} 
In this experiment we investigated the effect of RB type on our pool of architectures. We found that no single RB was superior for all systems but, for a given architecture, there were large variations in model performance with RB type. For example, vanilla+CBAM RBs were better than NeXt blocks in five of six cases within the ResNeXt/RBD framework. In comparison, CBAMs significantly harmed performance for \textit{Tucodec} variants. This may be a result of interference between the coarse-grained attention mechanism of the CBAM blocks and the highly specific attention in the RAB blocks. 

\subsubsection{Activation function} \label{sec:expt_activations}
\begin{table}[!htb]
    \let\center\empty
    \let\endcenter\relax
    \centering
    \resizebox{0.85\textwidth}{!}{\begin{tabular}{lcccc}
\toprule
\multirow{2}{*}{ \textbf{Model} }   & \multirow{2}{*}{ \textbf{PReLU} } & \multirow{2}{*}{ \textbf{ReLU} } & \multirow{2}{*}{ \textbf{GDN} } & \textbf{Relative Improvement}    \\

      &     &    &   & \textbf{over Backbone } \\ \midrule
Backbone     & \textbf{0.2309}     & 0.9857    & 0.2970   & 0.00\%     \\
ResNeXt3-27-1-vanilla+CBAM & 0.1948     & 1.0058    & \textbf{0.1900 }         & 17.71\%    \\
RDB3-3-8-vanilla    & \textbf{0.1958}     & 0.9887    & 0.2027   & 15.20\%    \\
RDB3-27-4-vanilla+CBAM     & 0.1893     & -         & \textbf{0.1865}   & 19.21\%    \\
RAB-4-NeXt   & \textbf{0.1917}     & 0.9992    & 0.2146   & 16.98\%    \\
GRDN-NeXt+CBAM      & 0.1893    & 1.0001    & \textbf{0.1689}   & 26.85\%    \\
Tucodec-NeXt        & \textbf{0.0890}     & 0.1624    & 0.1586   & 61.47\%    \\
Tucodec-NeXt+CBAM   & \textbf{0.1870}     & 0.2662    & 0.2539   & 19.02\%    \\
Tucodec-vanilla     & \textbf{0.0858}     & 0.0939    & 0.1212   & \textbf{62.86\% }   \\
Tucodec-vanilla+CBAM       & 0.3172     & \textbf{0.1788 }   & 0.2805   & 22.56\%    \\\bottomrule
\end{tabular}
}
    \caption{\label{tab:activations}DA MSE variation with activation function in models trained for 150 epochs. All models used PReLU activations applied channel-wise.}
\end{table} 
In this experiment we investigated the effect of activation function on the best performing systems from the RB experiments. As the \textit{Tucodec} models were performing well, we investigated the effect of the different activations on all four RBs\footnote{Note that, the \textit{Tucodec} model has three GDN activations in its core encoder design (as shown in \fref{fig:zhou2019}) which were present throughout all experiments. Here we changed the activations in the RBs and RABs only.} The results are shown in Table~\ref{tab:activations}.

In IC, Cheng et al.~\cite{Cheng2018}, found that PReLU activations were superior to ReLUs while Ma et al.~\cite{Ma2018} found that GDNs outperformed ReLUs. We believe we are the first to compare GDNs to PReLUs and found that the latter is superior. Crucially, we found that GDNs were very unstable as the function they compute allows for division by zero. In fact, in two of the three cases in Table~\ref{tab:activations} for which GDNs performed `best' the models actually produced \texttt{inf} predictions for one of the 197 test set samples. This process occurred more often earlier in training than later, and with the test data than with the training data but as they work well in the \textit{Tucodec} backbone, we reasoned that they are unstable when the input distribution is unpredictable.

\subsubsection{L1 fine tuning} \label{sec:l1_finetuning}
    \begin{figure}[!htb]
        \center{\includegraphics[width=0.9\textwidth]
        {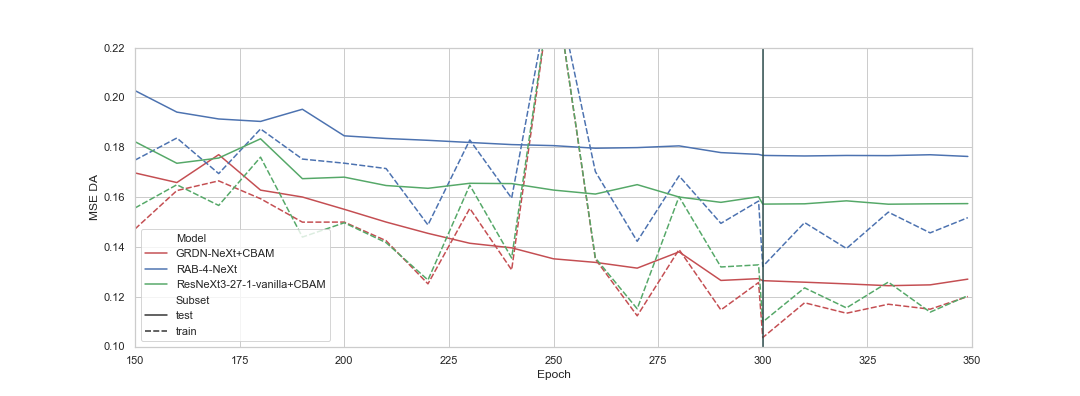}}
        \caption{\label{fig:fine_tuning}The MSE DA with L1 fine-tuning from epoch 300 onwards. The training-set DA MSE decreased but this was not accompanied with a test-set decrease. We investigated L1 fine tuning for all models in Table~\ref{tab:architecture_final} but just present a representative selection here.}
    \end{figure}

We experimented with the use of L1 fine-tuning late in the training process as recommended in \cite{Lu2019} but found that it did not give an appreciable benefit. In fact, as shown in \fref{fig:fine_tuning}, in our experiments it increased the degree of over fitting without providing any generalisation advantage. 

\subsubsection{Architecture Summary} 
\begin{table}[!htb]
\let\center\empty
\let\endcenter\relax
\centering
    \resizebox{0.9\textwidth}{!}{\begin{tabular}{lccc}
\toprule
\textbf{       Model}& \textbf{DA MSE} & \textbf{Execution Time (s)} & \textbf{Number of Parameters}\\\midrule
Backbone  & 0.1665 & 0.0897& 0.3M  \\
RDB3-27-4-vanilla+CBAM  & 0.1594 & 0.4666& 25.6M \\
ResNeXt3-27-1-vanilla+CBAM & 0.1548 & 0.1693& 3.5M\\
RAB-4-NeXt& 0.1723 & 0.1192& 1.3M\\
GRDN-NeXt+CBAM  & 0.1241 & 0.0983& 4.7M\\
Tucodec-vanilla & 0.0809 & 0.1294& 10.6M \\
Tucodec-NeXt & \textbf{0.0787} & \textbf{0.0537}& 2.5M\\
\bottomrule
\end{tabular}
 }
    \caption{Summary of the DA MSE and inference speeds of our best performing models after training to convergence.}
    \label{tab:architecture_final}
\end{table} 
We trained a selection of our best models to convergence and found that the Tucodec-NeXt and Tucodec-vanilla models performed best as shown in Table~\ref{tab:architecture_final}. The two models have very similar DA MSE values but the NeXt model is almost x2.5 faster during inference. As such, we use the Tucodec-NeXt model when making comparisons with reduced space DA in the following section.

%%%%%%%%%%%%%%%%%%%%%%%%%%%%%%%%%%%%
\section{Evaluation}\label{sec:expt_TSVDvsAE}

In this section, we compare our system against Reduced space VarDA with TSVD as described in \cite{Arcucci2019a}. Our system has superior DA performance on the test set as shown in Table~\ref{tab:svd_ae}. This is not just true on average: our system is \textit{consistently} better (see \fref{fig:comp_time}) and space (see \ref{appendi:comp_DA_SVD}). Moreover, we show in \fref{fig:tsvd_modes} that our method has a DA MSE that is 15\% lower than the reduced space approach with $M=n $ even in the limit in which the method becomes Parish et al.'s approach as there is no truncation of $\bm{V}$ ($\tau = S$). This is surprising: even if our CAE was truly lossless, which it is not, the matrices $\bm{V}$ and $\bm{V}_l$ contain the same information (albeit the latter stores it more efficiently)\footnote{As an aside, note that in this case our method is x43 faster than the traditional approach.}. The better performance of our bi-reduced space approach might be explained by the poor conditioning in the mono-reduced space and the resulting numerical errors. This requires further research.

\begin{table}[!ht]
\let\center\empty
\let\endcenter\relax
\centering
\resizebox{0.6\textwidth}{!}{\begin{tabular}{lccc}
\toprule
\multicolumn{1}{c}{\textbf{Model }} & \textbf{DA MSE} & \textbf{ Excecution Time (s)} \\ \midrule
Ref MSE  & 1.0001 & -   \\ 
\midrule
TSVD, $\tau = 32$, $M = n$ & 0.1270 & 1.8597  \\
TSVD, $\tau = 32$, $M = 0.1n$ & 0.1270 & 0.2627 \\
TSVD, $\tau = 32$, $M = 0.01n$ & 0.1334 & 0.0443 \\
TSVD, $\tau = 32$, $M = 0.001n$ & 0.1680 &  \textbf{0.0390}\\
\midrule
Tucodec-NeXt  & \textbf{0.0787} & 0.0537  \\
\bottomrule
\end{tabular}

%TSVD, $\tau = 32$, $M = 25$ & 0.1769 &  0.036 \\
%TSVD, $\tau = 32$, $M = 2$ & 0.3613 &  0.033 \\
}
\caption{\label{tab:svd_ae}Comparison of our best \textit{Tucodec} model with the Arcucci et al. approach \cite{Arcucci2019a} which sets $\sigma_{\tau}= \sqrt{\sigma_1}$ = 32. Our DA MSE is 37\% lower than the best Arcucci et al. system.} 
\end{table}

\begin{figure}[!htb]
        \center{\includegraphics[width=.9\textwidth]
        {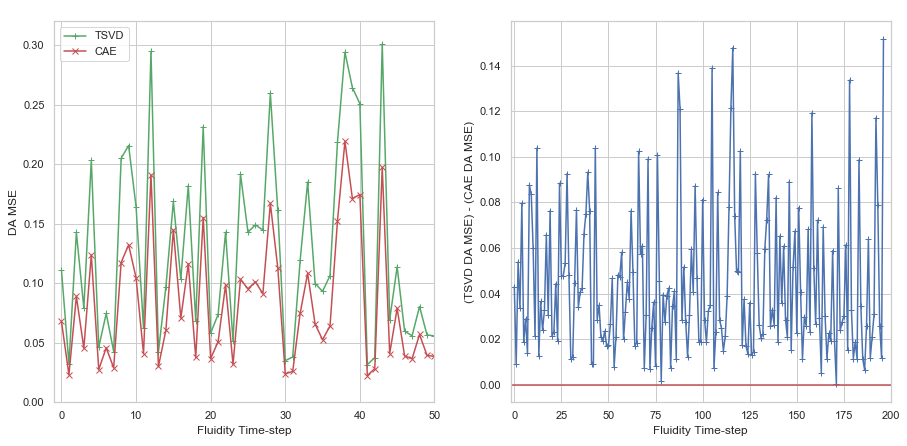}}
        \caption{\label{fig:comp_time}Comparison of TSVD ($\tau$ = 32, $M=n$) and AE data assimilation performance across sequential Fluidity time-steps. Note in Figure a) that the two methods find the same states difficult. In Figure b) we give the difference between the DA MSEs of the two methods for the whole test set. The proposed method performs better (is above the red line) in the vast majority of cases.}
      \end{figure}

\subsubsection{Performance-speed tradeoff}
\begin{figure}[!htb]
        \centering
        \includegraphics[width=\textwidth]
        {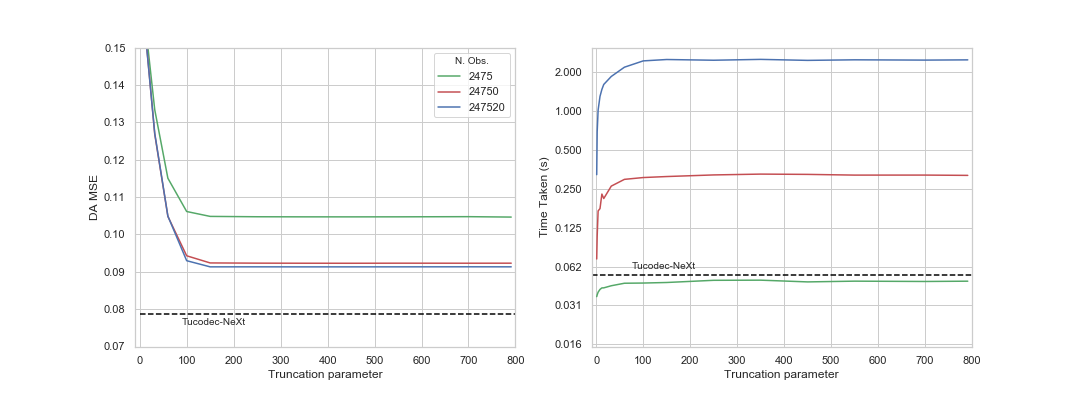}
        \caption{\label{fig:tsvd_modes}Effect of truncation parameter $\tau$ on a) DA MSE and b) online time. \textit{Tucodec} DA MSE and CPU execution time are marked with dashed black lines. Note the logarithmic scale on the y-axis in b). }
        \includegraphics[width=\textwidth]
        {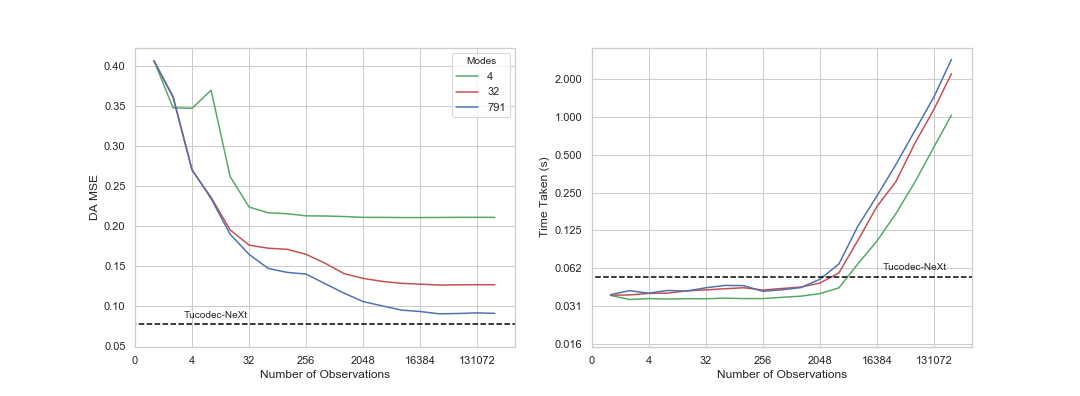}
        \caption{\label{fig:tsvd_obs}Effect of number of observations $M$ on a) DA MSE and b) execution time. The fact that the MSE is not monotonically increasing with modes=4 is due to the fact that we randomly choose a different subset of observations for each experiment. }
      \end{figure}

The reduced space approach in \cite{Arcucci2019a} has an acute performance-speed tradeoff occurring along three axes: 
\begin{enumerate}
    \item The size of the truncation parameter $\tau$. As this increases, the DA performance increases but the speed decreases as shown in \fref{fig:tsvd_modes}.  
    \item The number of observations $M$. As this increases, the DA performance increases but the speed decreases as shown in \fref{fig:tsvd_obs}. This is a stronger effect than 1. Noting the logarithmic scales in this \fref{fig:tsvd_obs}b, it is clear that there is only a small range of $M$ for which our method is slower than TSVD. 
    \item The size of the observation variance $\sigma_0$. We do not consider this here but \cite{Arcucci2019a} showed that as this parameter increases, performance increases but speed decreases. In all experiments here we used $\sigma_0 = \sigma_l =$ 0.005.
\end{enumerate}

Our system's evaluation speed is not sensitive to the number of observations\footnote{The performance of our system \textit{will} be affected by decreasing the number of observations.}, nor the value of $\sigma_l$. We show the performance-speed tradeoff for a range of models in \fref{fig:time_perf_tradeoff}. All timing measurements were carried out on the Intel Xeon E5-2690 v3 (Haswell) 2.60 GHz CPU and averaged over the test set. The clock was started at the beginning of the minimisation routine when all relevant data was already in memory. 
\begin{figure}[!htb]
    \centering
        \includegraphics[width=0.9\textwidth]
        {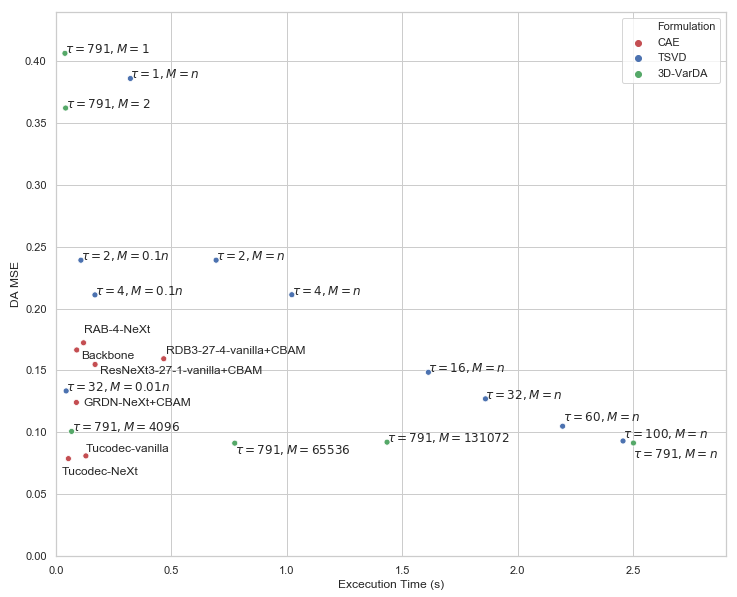}
        \caption{\label{fig:time_perf_tradeoff}Performance-speed tradeoff for a range of systems.}
      \end{figure}

%%%%%%%%%%%%%%%%%%%%%%%%%%%%%%%%%%%%
\section{Discussion}\label{sec:discussion}
In this section we briefly discuss a few points raised in the course of this research.
\subsection{Optimal $\tau$ and $M$}
Considering \fref{fig:time_perf_tradeoff}, it appears that, for the combinations of $\tau$ and $M$ that we investigate here, the most successful pairing in the performance-speed tradeoff is $M = 4096$ and $\tau = 791$ (i.e. no truncation). It is worth making a few observations on this result:
\begin{enumerate}
    \item There was no way to know that this combination was the best in advance as it is data-set dependent. It took ~60 CPU hours to calculate the DA MSE on the test-set for the range of $\tau$ and $M$ displayed in \fref{fig:time_perf_tradeoff}. We note that this value could be reduced by a more intelligent search method, but draw the reader's attention to the fact that this is of the same order as the 15 hours required to train a Tucodec-NeXt model to convergence on a GPU. 
    \item This result is still 30\% slower and 20\% less accurate for DA than the Tucodec-Next model.
\end{enumerate}

\subsection{Hardware Accelerators}
All timing comparisons were made on the CPU as we did not have a GPU implementation of Arcucci et al.’s routine. Using a K80 GPU with our method provided a speed-up of approximately 40\% for our method. This was with a \textit{very} poor implementation in which the data was transferred from the CPU to the GPU and back again. We expect an optimised implementation of our system on a modern accelerator to achieve a much larger relative speed-up over the figures here than the equivalent optimised version of Arcucci et al.’s method. A full defense of this claim might take another paper but we briefly sketch our our argument for this claim in the following paragraph. 

Recall that the bottleneck in our system is a fully convolutional decoder while, in the Arcucci et al. routine, the cost is dominated by large matrix multiplications and vector dot-products. The convolutional kernel parameters are used many times in a forward pass but there is almost no data reuse in the Arcucci et al. case. As such, the latter will be bandwidth-limited but the former may enter the compute-bound domain on some hardware-platforms\footnote{Convolution may be memory-bound depending on channel size and dimensions of the feature map.}. Historically, it has proven easier to accelerate compute-bound processes than memory bound ones and there is reason to believe that this will continue, not least because there is a whole industry built-around the design of systems that specifically accelerate convolutional inference workloads. We will not attempt to review the hardware options here but if the Graphcore `IPU', which is set to ship early in 2020, delivers on its marketing promises \cite{Lacey2018}, it would speed up our inference by up to three order of magnitude. There are also lower-cost, lower-power options such as FPGAs \cite{Chen2015}. As  such, we believe the quoted figures underestimate the latency advantage of our approach. 

\subsection{Other Acceleration Options}
During our architecture search, we optimised for DA performance rather than inference latency. Had we been focusing for the latter, there are a number of techniques aside from hardware acceleration that could be used to aid this. Firstly, a thinner decoder could be used as suggested in \cite{Theis2017} since only the decoder is evaluated in the online setting. Secondly, the existing network could be quantized \cite{Jacob2018} or pruned \cite{Li2016,Bellec2017} or both \cite{Han2016} to provide a substantial speed-increase. Additionally, convolutional acceleration approaches such as Pixel Shuffle \cite{Shi2016} or factorised convolutions \cite{Wang2018} might be employed to reduce the number of FLOPs in the forward decoder pass. Finally, there are innumerable small architectural changes that could be made in a similar vein to the replacement of vanilla blocks in the original \textit{Tucodec} model with NeXt blocks. None of the above strategies are available to traditional VarDA approaches. We note that some of these techniques will reduce the performance of our system but, as our approach has a considerable performance cushion over traditional approaches in its current form, this may be acceptable in some settings.

\subsection{Importance of Architecture}
The results in this paper demonstrate the central importance of using good CAE architectures. This field is moving exceptionally fast: our Backbone network, was state-of-the-art for image compression in 2017 \cite{Theis2017} but gives a DA MSE that is a) double that of the \textit{Tucodec} models and b) considerably poorer for DA than the Arcucci et. al. approach with $\tau = 32$ and $M=0.01$.

We found that it was non-trivial to extend many architectures to three spatial input dimensions and it required a large amount of manual tuning of the channel sizes so as not to create unreasonably large 4D feature maps (three spatial dimensions and one channel). In particular, our implementation of the GDRN \cite{Kim2019} had extreme computational requirements in 3D which, despite its modest number of parameters (see Table~\ref{tab:architecture_final}), took almost three times longer to train than any other network.

%%%%%%%%%%%%%%%%%%%%%%%%%%%%%%%%%%%%
\section{Conclusions and Future Work}\label{sec:conc_fw}

We have presented a new Bi-reduced space 3D-VarDA formulation and show that, in combination with the Zhou et al. or `\textit{Tucodec}' image compression CAE, this method gives superior data assimilation performance in comparison with reduced space VarDA regardless of the parameters used in the latter case. We have demonstrated that our method is also faster in the majority of scenarios. On the theoretical side, we show that our method produces approximately equivalent solutions to the traditional method at lower computational complexity. Unlike the previous approach which is in $\mathcal{O}(M^2)$ for large $M$, our method does not penalize the collection of more observation data. We have released our work in a well tested Python module \texttt{VarDACAE}.

There were many extensions to this work which we would have liked to explore further. We feel that the most important of these is the validation of our hypothesis that is possible to create an observation encoder network $f^o$ to calculate the latent misfits $\bm{d}_l$. We would also have liked to apply our approach to 4D-VarDA, validate it on other data sets and investigate alternatives to the L-BFGS minimization routine. A more substantial extension would involve integrating our method with CAE-based ROM approaches to produce a single end-to-end network for reduced space data assimilation and we believe this would be complemented by the use of data assimilation localization techniques \cite{Montmerle2018}. Finally, there is also potential for the use of VAEs within the proposed system to enforce orthogonality in the CAE latent dimension.

%%%%%%%%%%%%%%%%%%%%%%%%%%%%%%%%%%%%
\section*{Acknowledgements}
This work is supported by the EPSRC Grand Challenge grant ``Managing Air for Green Inner Cities'' (MAGIC) EP/N010221/1, by the EPSRC Centre for Mathematics of Precision Healthcare EP/N0145291/1 and the EP/T003189/1 Health assessment across biological length scales for personal pollution exposure and its mitigation (INHALE). Thanks to Dr. Laetitia Mottet for the set up of the full model in Fluidity.
M. Molina-Solana was supported by European Union's H2020 MSCA-IF (ga. No. 743623) and Athenea3i (ga. No. 754446) programmes.

%%%%%%%%%%%%%%%%%%%%%%%%%%%%%%%%%%%%
\bibliographystyle{elsarticle-num}
\bibliography{DA_report}

\begin{appendices}

\appendix
%\chapter{Appendices}
%\counterwithout{section}{chapter}
\section{Architecture Search Details}\label{sec:appendix_architecture}
In this appendix, we give details of our architecture search that would be out of place in the main text. 
\subsection{Residual Block diagrams}
\begin{figure}[!htb]
        \centering
        \center{\includegraphics[width=\textwidth]
        {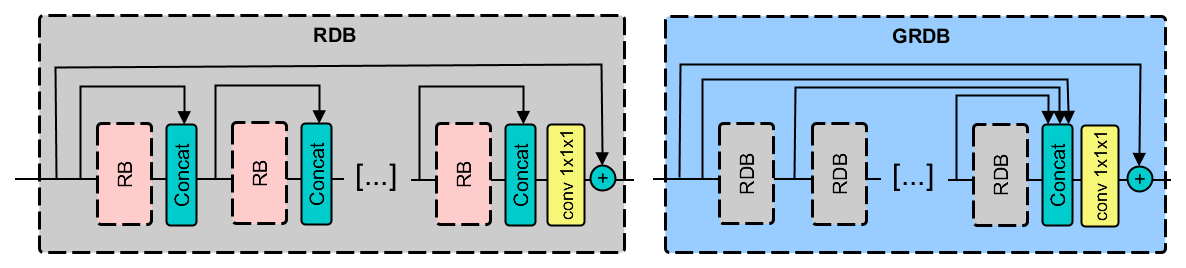}
        \caption{\label{fig:RDB_GRDB}a) The Residual Dense Block \cite{Huang2017}, \cite{Zhang2019} and its extension b) The Grouped Residual Dense Block \cite{Kim2019}.} 
        \includegraphics[width=\textwidth]
        {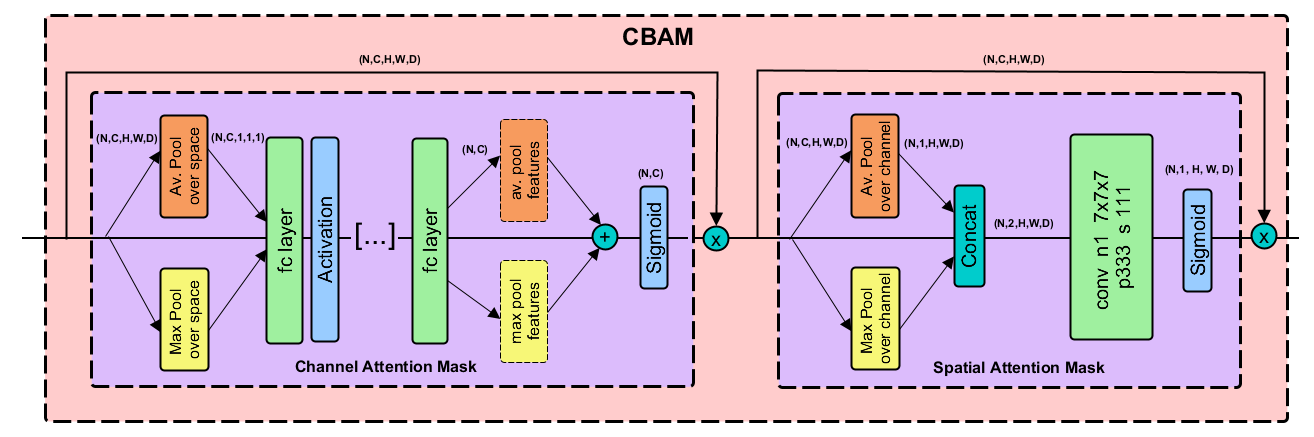}
        \caption{\label{fig:CBAM}The Convolutional Block Attention Module \cite{Woo2018}. The channel mask $M_C(\bm{x})$ and spatial mask $M_s(\bm{x})$ are applied sequentially. These masks are broadcast to full dimensions $(N, C, H, W, D)$ before their element-wise multiplication with the inputs $\bm{x}$. Note that in $M_C(\bm{x})$, the features from max pooling and average pooling are fed through the same fully connected network one after the other and the results are then added. `conv  n1  7x7x7 p333   s111' represents a convolutional layer with 1 channel, kernel size = (7, 7, 7), padding = (3,3,3) and stride = (1,1,1) and is specific to our implementation (although in some cases we found that kernel size = (3, 3, 3) was necessary to enable efficient training). Our CBAM has just two fully connected layers in $M_C(\bm{x})$.} 
        \includegraphics[width=0.8\textwidth]
        {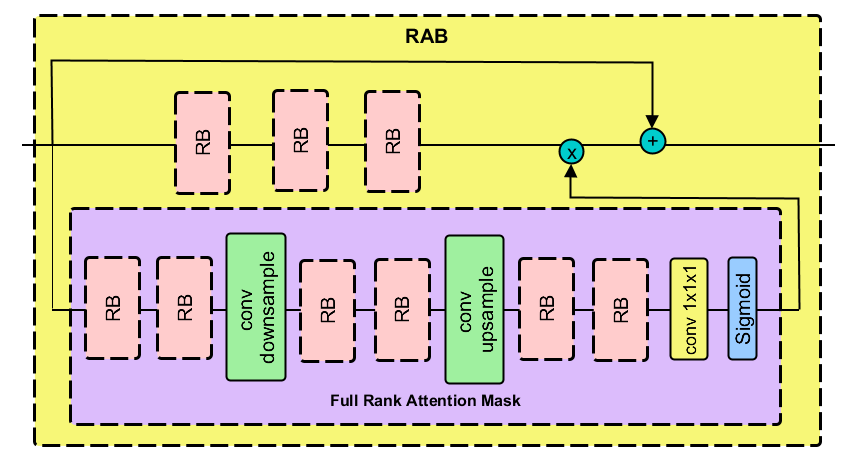}
        \caption{\label{fig:RAB}The Residual Attention Block proposed by \cite{Zhang2019} and utilised by \cite{Zhou2019}. Note that unlike CBAMs, the trunk of RABs (yellow background) are not the identity mapping.}}
\end{figure}
\begin{figure}[!htb]
        \center{\includegraphics[width=.8\textwidth]
        {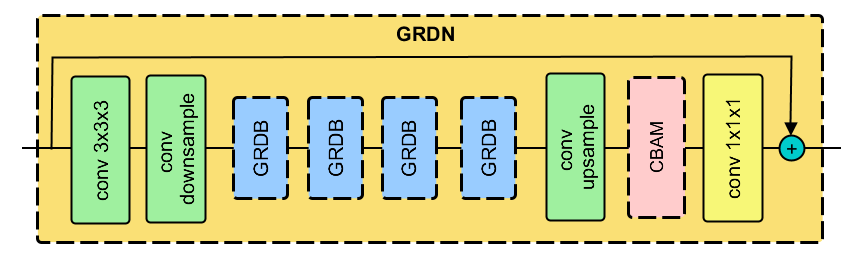}}
        \caption{\label{fig:GRDN}The GRDN block \cite{Zhang2020} in a series of GRDBs (see in \fref{fig:RDB_GRDB}) are used with a CBAM module.}
\end{figure}
The RAB-L variant is investigated in an attempt to separate the \textit{Tucodec} model's success from its use of RABs and the GRDN model follows the work of \cite{Cho2019a}. 

\subsection{ResNext Variant}
\begin{figure}[!htb]
        \center{\includegraphics[width=\textwidth]
        {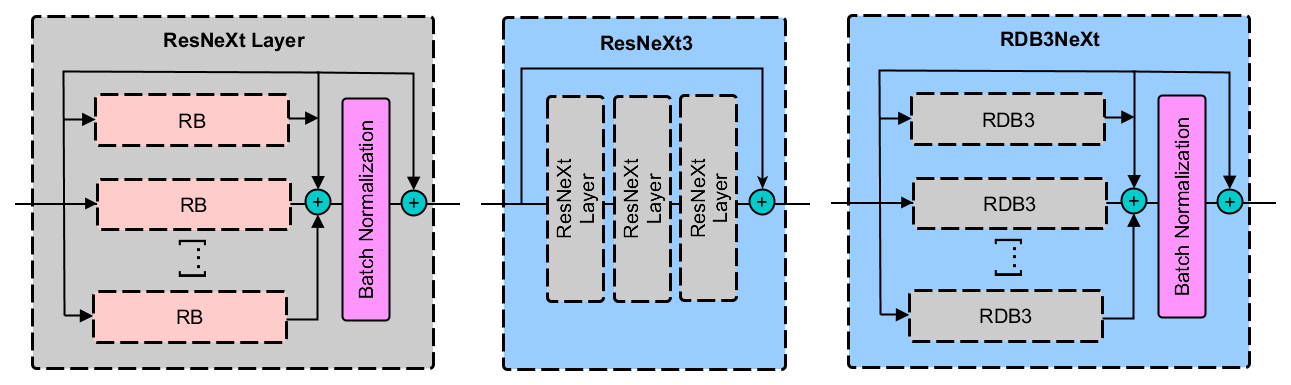}}
        \caption{\label{fig:ResNeXtLayers}\textbf{a)} A single ResNeXt layer, repeated from \fref{fig:ResNeXt_orig} for clarity \cite{Xie2017a}. The ResNeXt cardinality describes the number of RBs in each layer, \textbf{b)} three stacked ResNeXt layers with an extra residual connection, and \textbf{c)} a ResNeXt layer with RDBs instead of simple RBs. Each `RDB3' has 3 RBs.}
        \center{\includegraphics[width=0.8\textwidth]
        {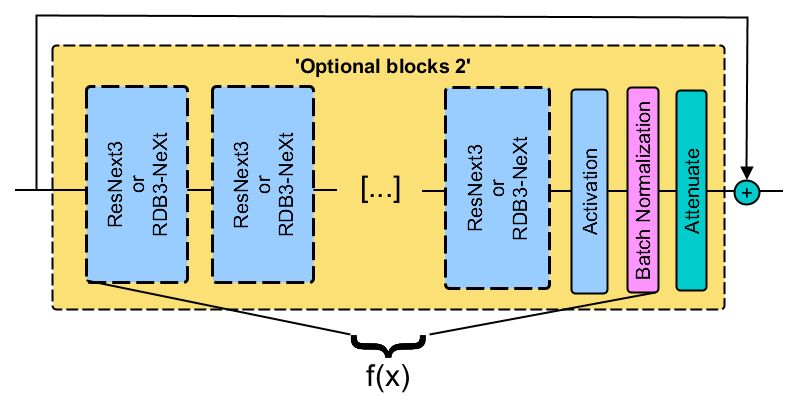}}
        \caption{\label{fig:ResNeXt}Our ResNeXt variant in which ResNeXt layers are grouped in threes. The residual attenuation coefficient in green is applied before leaving the block.}
      \end{figure}
      
We found that placing a flexible variant of the ResNeXt system \cite{Xie2017a} within the second optional block of our backbone was sufficient to describe almost all of the top non-Tucodec-based CLIC entries. In order to include the Chen et al. 2018 entry which used RDBs \cite{Chen2018}, each with three RBs, we extended the ResNeXt system to allow these building blocks as shown in \fref{fig:ResNeXtLayers}c). In order to make our vanilla ResNeXt variants comparable with these `RDB3s', we added an extra skip connection over every third ResNeXt layer as in \cite{Mentzer2018}. 

Within this system we refer to an architecture as: \\ $$\text{`RBD3NeXt-L-N-RB' \ \  \  or  \  \ \  `ResNeXt3-L-N-RB'}$$ \\
for an encoder that consists of the ResNeXt variant in \fref{fig:ResNeXt} with L layers each of cardinality N arranged in either the ResNeXt3 or the RBD3NeXt structure with residual blocks of type RB all embedded within the second optional block of our backbone in \fref{fig:backbone}. When the backbone network is included, these encoders have (L + 7) layers.\\

In this way, Chen et al.'s encoder can be described as a `RBD3NeXt-8-1-vanilla' \cite{Chen2018} while Mentzer et al.'s is a `ResNeXt3-27-1-vanilla' \cite{Mentzer2018}. By placing the CLIC entries within this structure, the landscape \textit{between} the entries in Table~\ref{tab:CLIC_res} became available to search. We evaluated the grid search of options within this space and find that 27 layers of width 4 RDB3s blocks (with CBAMs) perform best. This design is dissimilar to any CLIC entry meaning we would not have found it by simply following examples in the literature. 

\textbf{Attenuation coefficient}
We found that it was difficult to train ResNeXt variants with large cardinality but, as the backbone trained easily, it was clear that the new residual blocks were interfering with the backbone's ability to down-sample the inputs. Therefore we introduced a residual attenuation coefficient $\alpha$ at the exit to the block shown in \fref{fig:ResNeXt} such that the computed function is:
\begin{equation}
    g(\mathbf{x}) = \mathbf{x} + \alpha  \ f(\mathbf{x})
\end{equation} 
$\alpha$ was initialised to 0.05 at the start of training and then updated with the other network parameters. This down-weights the ResNext block's importance initially so that the backbone has time to learn a good compression. 

\textbf{ResNeXt width and cardinality}
\begin{figure}[!htb]
        \center{\includegraphics[width=0.8\textwidth]
        {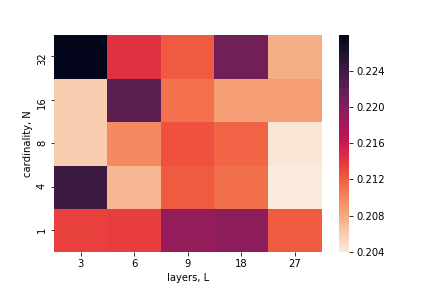}}
        \caption{\label{fig:ResNext_struct}The DA MSE heatmap for a ResNeXt3-L-N-NeXt models with a range of cardinalities and layers. All use PReLU activations.}
      \end{figure}

In the first experiment, we investigated the effect of changing cardinality and number of layers within residual component of our ResNeXt variant. We used NeXt RBs for these experiments. The results are in \fref{fig:ResNext_struct}. The best system was ResNeXt3-27-4-NeXt but there was also an interesting group of models with three layers on the left of \fref{fig:ResNext_struct}. We decided to preserve this diversity and take three models forward to the next stage of experiments. These were ResNeXt3-27-4-NeXt, ResNeXt3-3-8-NeXt and ResNeXt3-27-1-NeXt\footnote{This final model did not perform particularly well but is a ResNet-34 with the final linear layer removed and, as we knew we would be evaluating the systems with vanilla RBs, we thought that in view of its historical successes, there was a good argument for keeping this architecture.}.

\textbf{RAB-L} \label{sec:expt_RAB}
\begin{table}[!htb]
\let\center\empty
\let\endcenter\relax
\centering
\resizebox{0.7\textwidth}{!}{\begin{tabular}{ccc}
\toprule
\multirow{2}{*}{ \textbf{Number of RABs} } & \multirow{2}{*}{\textbf{ DA MSE }} & \textbf{Relative Improvement} \\
& & \textbf{over Backbone }\\
\midrule
1              & 0.2005 & 13.17\%                 \\
2              & 0.2188 & 5.24\%                  \\
4              & 0.1917 & 16.98\%                 \\
8              & 0.2071 & 10.31\%                 \\
\bottomrule
\end{tabular}

}
\caption{\label{tab:RAB}The DA performance of the RAB-L architectures.}
\end{table} 

In our second experiment we investigated the effect changing the number of consecutive RAB blocks in the backbone network. The results are in Table~\ref{tab:RAB}. The the best of these systems with four RABs, gives a 17\% improvement relative to the Backbone but is poor in comparison with the best \textit{Tucodec} model.

\subsubsection{Augmentation}\label{sec:expt_augmentation}
\begin{table}[!htb]
\let\center\empty
\let\endcenter\relax
\centering
\resizebox{0.6\textwidth}{!}{\begin{tabular}{cccc}
\toprule
\textbf{Augmentation}  & \textbf{Jitter}    & \textbf{Jitter}    & \textbf{Jitter Amplitude}   \\
\textbf{Strength}      & \textbf{Amplitude} & \textbf{Frequency} & \textbf{per Location}       \\
\midrule
0             & None      & None      & None               \\
1             & 0.005     & 0.5       & 0.0025             \\
2             & 0.05      & 0.25      & 0.0125             \\
3             & 0.1       & 0.5       & 0.0500             \\
\bottomrule
\end{tabular}
}
\caption{\label{tab:augmentation_strengths}The field-jitter augmentation strengths we investigated. We added Gaussian noise with standard deviation of `Jitter Amplitude' at `Jitter Frequency' of the total locations in the state.}
\end{table} 
To quantify what, if any, effect our augmentation technique was having, we retrained the Tucodec-NeXt model with a range of augmentation strengths as shown in \fref{fig:augmentation}. We did not observe a large difference between the methods so choose the strongest augmentation that did not harm performance (augmentation strength 2 in Table~\ref{tab:augmentation_strengths}) when training our models to convergence. 
\begin{figure}[!htb]
        \center{\includegraphics[width=\textwidth]
        {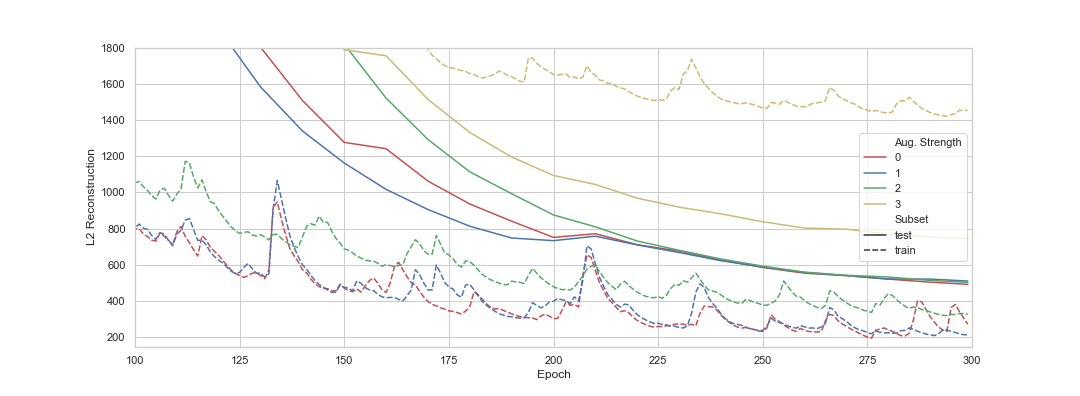}}
        \caption{\label{fig:augmentation}The training and validation MSE reconstruction errors during training with different strengths of augmentation detailed in Table~\ref{tab:augmentation_strengths}. These graphs have been smoothed with an exponential moving average with $\alpha=0.4$ as the spikes in the training curves made this diagram too noisy to be illustrative. A non-smoothed version is given in Appendix \ref{appendix:augmentation}.}
      \end{figure}

\section{Further Comparisons}\label{appendi:comp_DA_SVD}
This Appendix contains two graphs that would have been repetitious in the full text but provide useful context to the comparison between reduced space VarDA and bi-reduced space VarDA. 
\begin{figure}[!htb]
\center{\includegraphics[width=0.8\textwidth]
        {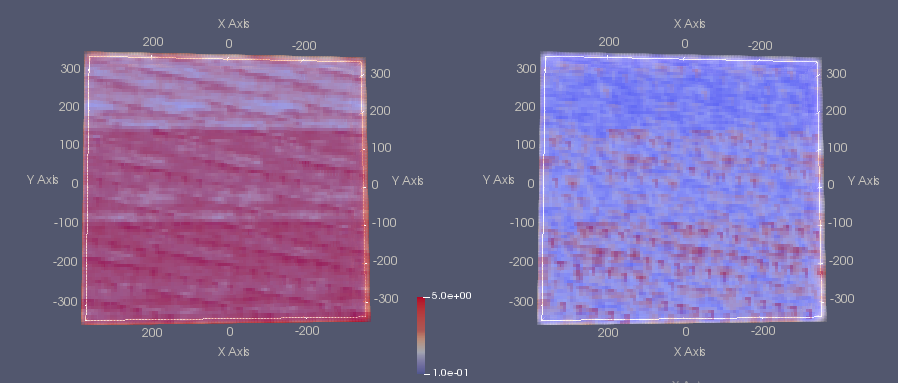}}
        \caption{\label{fig:comp_space}DA MSE across a slice of the spatial domain averaged over all test-set time-steps. We show a) the reduced-space variant with TSVD ($\tau$ = 32 and $M=n$) and b) Bi-reduced space variant with the Tucodec-Next model.}
\center{\includegraphics[width=0.95\textwidth]
        {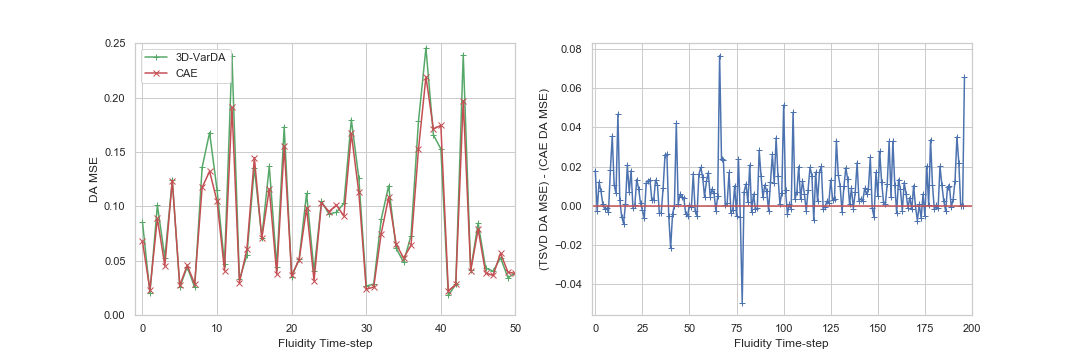}}
        \caption{\label{fig:comp_time791}Repetition of \fref{fig:comp_time} with $\tau = 791$ (i.e. no truncation) instead of $\tau = 32$  and $M= n$ as before. Although the performance is more similar in this case, our method still performs better on average. We also note that, in order to achieve this level performance, the reduced space method takes 2.5s, or x43 longer than our approach.}
\end{figure}
\end{appendices}

%%%%%%%%%%%%%%%%%%%%%%%%%%%%%%%%%%%%

\end{document}